\documentclass[journal]{IEEEtran} 
\usepackage{amsmath,amsfonts}

\usepackage{array}
\usepackage[caption=false,font=normalsize,labelfont=sf,textfont=sf]{subfig}
\usepackage{textcomp}
\usepackage{stfloats}
\usepackage{url}
\usepackage{verbatim}
\usepackage{graphicx}
\usepackage{cite}
\usepackage{algorithm2e}
\RestyleAlgo{ruled}

\usepackage{amsthm}

\newtheorem{definition}{Definition}

\newtheorem{prop}{Proposition}

\usepackage{hyperref}
\hypersetup{
    colorlinks=true,
    linkcolor=blue,
    filecolor=magenta,      
    urlcolor=magenta,
    pdftitle={Overleaf Example},
    pdfpagemode=FullScreen,
    }

\begin{document}
\title{MMP++: Motion Manifold Primitives with Parametric Curve Models}
\author{Yonghyeon Lee$^1$
\thanks{$^1$Center for AI and Natural Sciences (CAINS), Korea Institute for Advanced Study (KIAS), Seoul, South Korea, ylee@kias.re.kr}
\thanks{The paper will appear in the {\it IEEE Transactions on Robotics}. © 20xx IEEE.  Personal use of this material is permitted.  Permission from IEEE must be obtained for all other uses, in any current or future media, including reprinting/republishing this material for advertising or promotional purposes, creating new collective works, for resale or redistribution to servers or lists, or reuse of any copyrighted component of this work in other works. IEEE Xplore: \href{https://ieeexplore.ieee.org/document/10637485/}{LINK}. Digital Object Identifier (DOI): 10.1109/TRO.2024.3444068.}
}



\maketitle

\begin{abstract}
Motion Manifold Primitives (MMP), a manifold-based approach for encoding basic motion skills, can produce diverse trajectories, enabling the system to adapt to unseen constraints. Nonetheless, we argue that current MMP models lack crucial functionalities of movement primitives, such as temporal and via-points modulation, found in traditional approaches. This shortfall primarily stems from MMP's reliance on discrete-time trajectories.
To overcome these limitations, we introduce Motion Manifold Primitives++ (MMP++), a new model that integrates the strengths of both MMP and traditional methods by incorporating parametric curve representations into the MMP framework. Furthermore, we identify a significant challenge with MMP++: performance degradation due to geometric distortions in the latent space, meaning that similar motions are not closely positioned. To address this, Isometric Motion Manifold Primitives++ (IMMP++) is proposed to ensure the latent space accurately preserves the manifold's geometry.
Our experimental results across various applications, including 2-DoF planar motions, 7-DoF robot arm motions, and SE(3) trajectory planning, show that MMP++ and IMMP++ outperform existing methods in trajectory generation tasks, achieving substantial improvements in some cases. Moreover, they enable the modulation of latent coordinates and via-points, thereby allowing efficient online adaptation to dynamic environments.

\end{abstract} 


\section{Introduction}
Developing ``good'' mathematical models for representing basic motion skills continues to be a central focus in the literature on learning from demonstration~\cite{argall2009survey, saveriano2021dynamic, zhu2018robot}. In this paper, we adopt the view that a good model should be capable of generating diverse trajectories that can complete the given task. Moreover, it should be easily adaptable to a new, unseen constraint. For instance, if an unseen obstacle blocks the initially planned path, the model should enable a robot to avoid that obstacle while still accomplishing the task. We aim to train such a model using multiple demonstration trajectories. Challenges often arise from the small dataset size, high dimensionality of the trajectory data, and the multi-modality of data distribution.

Adopting the motion manifold hypothesis~\cite{arvanitidis2017latent, lee2023geometric} -- which assumes that a set of high dimensional trajectory data lies on some lower-dimensional manifold --, recent {\it Motion Manifold Primitives (MMP)} framework provides motion primitive models that can encode and generate, for a given task, a continuous manifold of trajectories each of which is capable of accomplishing the task~\cite{noseworthy2020task, lee2023equivariant}. 
This framework has demonstrated promising results in addressing the aforementioned challenges, effectively reducing the data dimensionality and capturing multi-modality. 
In particular, adjusting the low-dimensional latent coordinate values enables the adaptation of trajectories to unseen environments.

\begin{figure}[!t]
    \centering
    \includegraphics[width=0.9\linewidth]{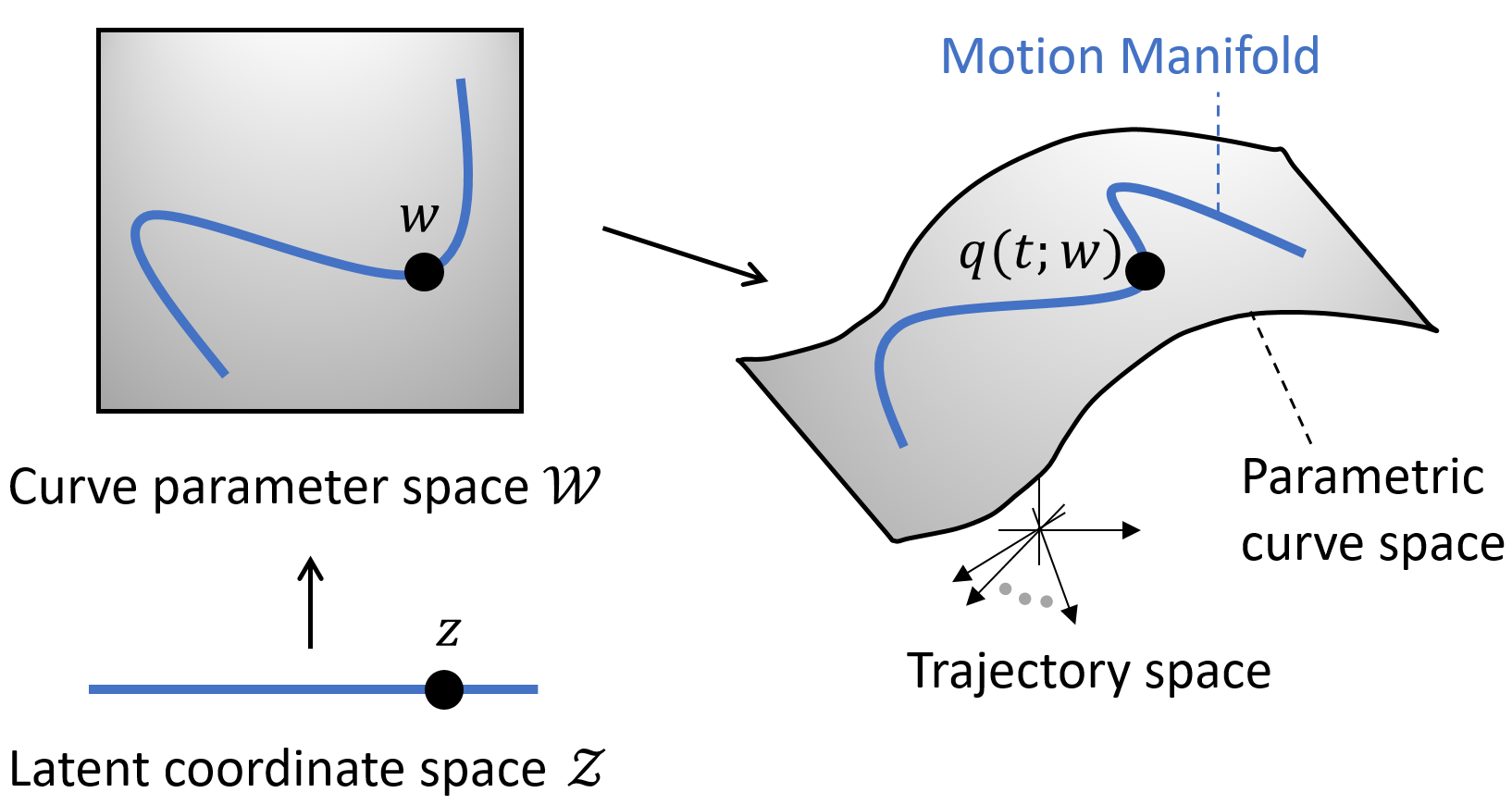}
    \caption{{\it MMP++}: A latent coordinate space ${\cal Z}$ is mapped to a subspace of the curve parameter space ${\cal W}$; the parameter space ${\cal W}$ is mapped to a subspace of the infinite-dimensional trajectory space.
    The motion manifold and parametric curve space are visualized as a curve and surface, not because their actual dimensions are one and two, but only to indicate the relative size relationships of their dimensions.} 
    \label{fig:intro2}
    \vspace{-5pt}
\end{figure}

These manifold-based models, however, lack some of the desired functionalities of movement primitives found in other conventional methods such as DMP~\cite{schaal2007dynamics, ijspeert2013dynamical}, ProMP~\cite{paraschos2013probabilistic}, and VMP~\cite{zhou2019learning}. These functionalities include: (i) temporal modulation to enable faster or slower execution of the movement and (ii) modulation of via-points (e.g., start and goal points) given new task constraints.
The fundamental reason for the absence of such functions in the MMP framework is its reliance on discrete-time trajectory representations.
In contrast, conventional movement primitives often employ parametric models for trajectory representation. For example, one of the simplest forms of these models is the linear basis function model, for a configuration space ${\cal Q}=\mathbb{R}^{n}$, expressed as:
\begin{equation}
    q(\tau; w) = \sum_{i=1}^{B} \phi_i(\tau) w_i \:\:\: \text{for} \:\:\: \tau \in [0,1],
\end{equation}
where $\{\phi_i(\tau)\}_{i=1}^{B}$ is a set of some scalar-valued basis functions in $[0,1]$ and $w_i \in \mathbb{R}^{n}, i=1,\ldots,B$ are curve parameters. 
Temporal modulation can be achieved by modifying $\tau$ as a function of time $t$, and adding some structures to the basis function $\phi_i(\tau)$ can enforce constraint satisfaction (e.g., $\phi_i(\tau):=\tau (1-\tau) b_i(\tau)$ enforces $\phi_i(0)=\phi_i(1)=0$ for any function $b_i(\tau)$).

In this paper, we propose applying the MMP framework to the parametric curve representations of trajectories, thus simultaneously tackling the challenge of dimensionality and achieving the desired functionalities, denoted as {\it MMP++} (see Fig.~\ref{fig:intro2}). 
Additionally, parametric curve representations in MMP++ lead to several other advantages. First, motions have bounded accelerations and jerks, avoiding sudden and abrupt changes. Second, the dimension of the parametric curve space -- which is equal to the dimension of the parameter space -- is generally much smaller than the dimension of the discrete-time trajectory data space, reducing the complexity of the subsequent motion manifold learning problem.

\begin{figure*}[!t]
    \centering
    \includegraphics[width=1\textwidth]{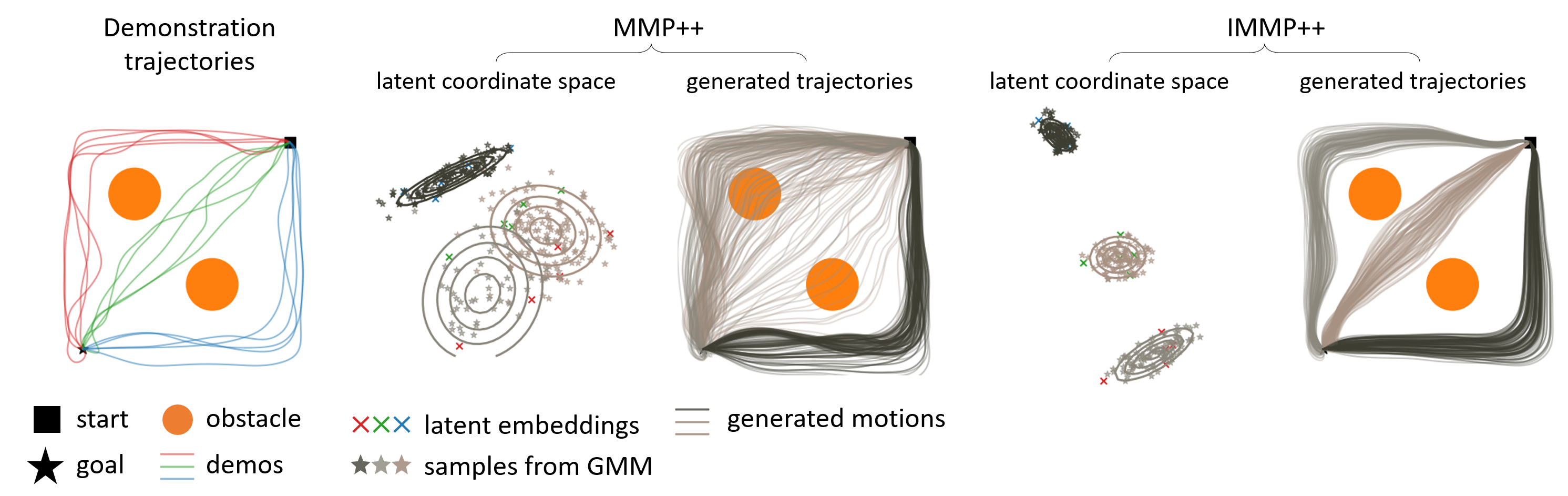}
    \caption{\textit{Left}: There are 15 demonstration trajectories (red, green, and blue trajectories) that travel from the start to the goal, avoiding the obstacle. \textit{Middle and Right}: MMP++ and IMMP++ learn two-dimensional manifolds in the curve parameter space and produce two-dimensional latent coordinate spaces. Latent values of the demonstration trajectories are visualized in the latent coordinate spaces, marked as $\times$. GMMs of three components are fitted in the latent spaces, and the sampled points are visualized as stars $*$. The corresponding generated trajectories are also visualized.} 
    \label{fig:intro}
    \vspace{-5pt}
\end{figure*}


The vanilla MMP++, a naive application of the MMP framework to the curve parameter space, however, sometimes results in a {\it geometrically distorted} latent coordinate space -- where similar motion data are not positioned close to each other --, leading to a generation of motions that violate the task constraint. 
For example, consider an example shown in Fig.~\ref{fig:intro}.
We learn a two-dimensional manifold and its latent coordinate space, by using the red, green, and blue demonstration trajectories and their parametric curve representations. 
Then we fit a Gaussian Mixture Model (GMM) using the latent values of the trajectories.
As illustrated in Fig.~\ref{fig:intro} (\textit{Middle}), due to the geometric distortion in the latent coordinate space, the same color trajectories are not located close enough to each other. Consequently, the red and green latent points are not correctly clustered by the GMM, and many of the generated motions collide with the obstacle, failing to accomplish the task.

In this paper, adopting the isometric regularization method from~\cite{yonghyeon2022regularized}, we propose to learn {\it a geometry-preserving latent coordinate space}, so that similar trajectories can be located nearby in the latent space. 
To employ this method in our context, we need to specify a {\it Riemannian metric} for the curve parameter space that serves as the basis for determining the notion of closeness in the parameter space. 
We propose a {\it CurveGeom Riemannian metric} for the curve parameter space, that reflects the geometry of the trajectory space, given a parametric curve model that satisfies some mild regularity conditions.
We call this framework {\it Isometric Motion Manifold Primitives++ (IMMP++)}; see Fig.~\ref{fig:intro} ({\it Right}).

In the first part of our experiments, we focus on Euclidean configuration space cases and use affine curve models, similar to those in ProMP~\cite{paraschos2013probabilistic} and VMP~\cite{zhou2019learning}. This induces constant CurveGeom metrics and leads to simpler implementations of the isometric regularization. Experiments involving 2-DoF planar obstacle-avoiding motions and 7-DoF collision-free motions of a robot arm confirm that our manifold-based methods, MMP++ and IMMP++, outperform conventional movement primitives in trajectory generation. Notably, we verify that the modulation of latent coordinates and via-points leads to diverse trajectory generation and enables the online adaptation of trajectories in dynamic environments.  In the second part, we demonstrate how to extend our framework to SE(3) trajectory data with the water-pouring demonstration trajectory dataset.
\section{Related Works}
\label{sec:rw}

\subsection{Movement Primitives}
Movement Primitives are mathematical models that encode and generate motions or trajectories. Conventional methods for movement primitives can be roughly divided into two categories: (i) dynamical system-based approaches, such as Dynamic Movement Primitives~\cite{ijspeert2013dynamical, schaal2007dynamics, pervez2017novel, fanger2016gaussian, umlauft2017bayesian, pervez2017learning} and Stable Dynamical Systems~\cite{khansari2011learning, neumann2013neural, neumann2015learning, blocher2017learning, sindhwani2018learning, kolter2019learning}; and (ii) parametric or non-parametric probabilistic modeling of trajectories~\cite{paraschos2013probabilistic, zhou2019learning, huang2019kernelized}. Each of these models possesses its own characteristics. Dynamical system-based methods typically ensure the stability of the resulting closed-loop systems. Approaches based on stable autonomous dynamical systems provide temporal and spatial robustness to perturbations. Parametric curve models offer adaptability in terms of temporal modulation and changes in constraints (e.g., goal points). 

A primary challenge in most of these methods is the limited adaptability to diverse situations (e.g., when unforeseen obstacles or new constraints emerge), which is mainly due to their design for encoding and producing a single trajectory for a given task~\cite{lee2023equivariant}.  
This becomes problematic when this trajectory becomes infeasible by unforeseen environmental changes, such as the sudden appearance of an obstacle. 
While dynamical system-based approaches can incorporate mechanisms like obstacle avoidance potential functions~\cite{park2008movement, hoffmann2009biologically, khansari2012dynamical, ginesi2019dynamic}, these adaptations may inadvertently breach other task-related constraints. Therefore, for motion primitives to be truly adaptable, a strategy that can encode various trajectories for the same task is critical. 
Our work adopts the motion manifold primitives framework~\cite{noseworthy2020task, lee2023equivariant} to encode and generate diverse trajectories or even a continuous manifold of trajectories, producing highly adaptable primitives, simultaneously inheriting advantages of the parametric curve models utilized in~\cite{paraschos2013probabilistic, zhou2019learning}.

\subsection{Manifold-based Movement Primitives}
Recently, manifold-based representations of basic motion skills have shown promising results in encoding diverse motions and producing adaptable primitives. These approaches can be categorized into two types. The first approach attempts to learn a sub-manifold in the configuration space ${\cal Q}$, a manifold of configurations, where the latent value $z$ maps to a configuration in ${\cal Q}$~\cite{beik2021learning}. To generate a trajectory, this approach computes a geodesic connecting two points in ${\cal M}$. 

The second type learns a low-dimensional manifold of trajectories, referred to as a motion manifold, where each point in the latent space $z$ corresponds to a trajectory $q(t) \in {\cal Q}$~\cite{noseworthy2020task,lee2023equivariant}.
Our method extends the latter approach. 
While existing methods map the latent point $z$ to a discrete-time trajectory representation $(q_1,\ldots, q_T)$, we use parametric curve models $q(t, w)$, in which the latent point $z$ is mapped to the curve parameter $w$. 
As a result, our extended version can also be interpreted as learning a sub-manifold in ${\cal Q}$, similar to the first type approach, since $(t,z)$ is now mapped to a point $q(t, w(z))$ in ${\cal Q}$.

Motion manifold primitives can produce diverse trajectories, and their ability to adapt to unseen obstacles by finding a latent value that generates a collision-free trajectory has been verified~\cite{lee2023equivariant}. However, due to the discrete nature of trajectory representation, they have limited adaptability to a dynamically changing environment. In our extended version, trajectories are parameterized by time, enabling online adaptation to dynamic environments.

\subsection{Manifold Learning and Latent Space Distortion} 
An autoencoder framework and its variants have received a lot of attention as effective methods to learn the manifold and its coordinate chart simultaneously, including but not limited to~\cite{lee2021neighborhood, kingma2013auto, creswell2017adversarial, rifai2011contractive, yonghyeon2022regularized, lee2023explicit, nazari2023geometric, yoon2021autoencoding,janggeometrically,lee2022statistical}. Of particular relevance to our paper, a geometric perspective on autoencoders has been eloquently presented in~\cite{lee2023geometric}.  
In our paper, a significant aspect of concern is the presence of the geometric distortion within the latent space of the autoencoder, as highlighted in~\cite{shao2018riemannian, arvanitidis2017latent, yonghyeon2022regularized, lee2023geometric, nazari2023geometric}. 
A recent regularization method~\cite{yonghyeon2022regularized} has developed a method to find the one that minimizes the geometric distortion, i.e., preserves the geometry of the data manifold and the latent coordinate space.
While Euclidean metric is assumed in~\cite{yonghyeon2022regularized}, in this work, we propose to use a pullback Riemannian metric for the curve parameter space that reflects the geometry of the curve space. 

\section{Geometric Preliminaries}
\label{sec:rgpc}
In this section, we review some basic concepts in differential geometry that serve as cornerstones for our method. We refer to standard differential geometry textbooks for more details~\cite{do1992riemannian,fecko2006differential}.

\subsection{Riemannian Manifolds}

\begin{figure}[!t]
    \centering
    \includegraphics[width=\linewidth]{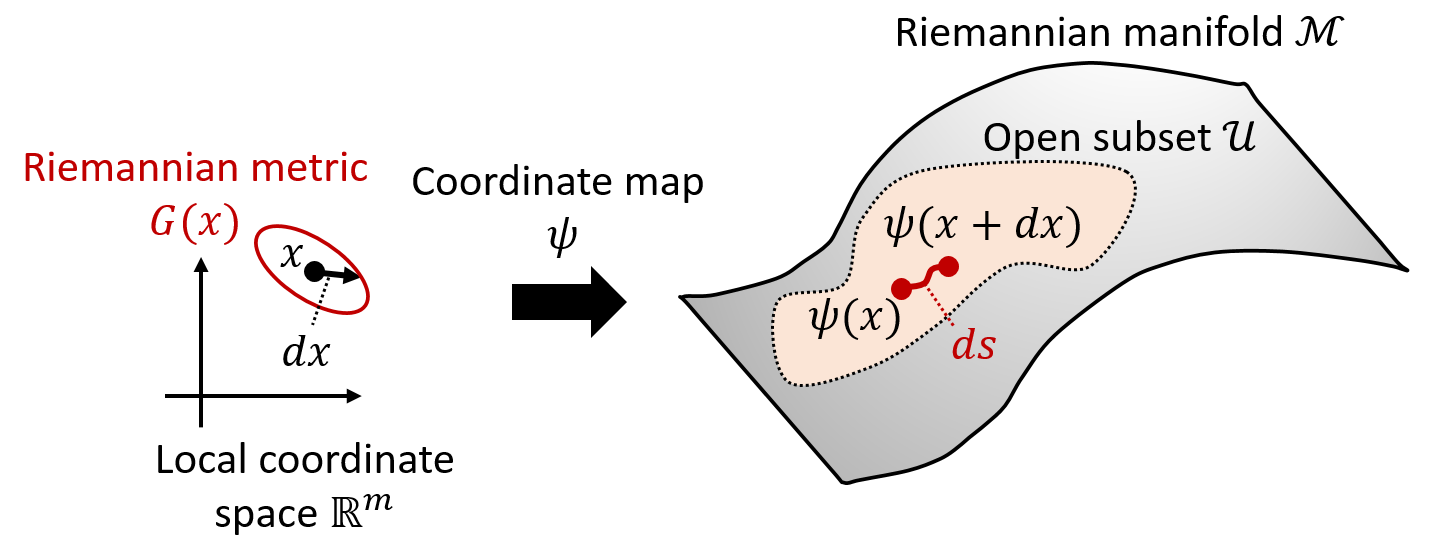}
    \caption{A local coordinate system for an $m$-dimensional Riemannian manifold ${\cal M}$. The Riemannian metric at coordinates $x$, $G(x)$, is visualized as a red equidistant ellipse that is $\{y \in \mathbb{R}^{m} \: | \: (y-x)^T G(x) (y-x) = {\rm constant}\}$.}
    \label{fig:gp1}
\end{figure}
A smooth manifold ${\cal M}$ equipped with a positive-definite inner product on the tangent space at each point is called a {\it Riemannian manifold}, and the family of inner products is called a {\it Riemannian metric}.
Given an $m$-dimensional Riemannian manifold ${\cal M}$ and its local coordinates $x \in \mathbb{R}^{m}$ -- when using local coordinates, we implicitly assume there exists a local coordinate map $\psi:\mathbb{R}^{m} \to {\cal U} \subset {\cal M}$ --, the Riemannian metric at $x$ can be expressed as an $m\times m$ positive-definite matrix denoted by $G(x) \in \mathbb{R}^{m \times m}$; see Fig.~\ref{fig:gp1}. This defines several geometric notions on ${\cal M}$, such as lengths, angles, and volumes. For example, given an infinitesimal displacement vector $dx \in \mathbb{R}^{m}$, its squared length is defined as follows:
\begin{equation}
    ds^2 = dx^T G(x) dx = \sum_{i,j=1}^{m} g_{ij}(x)dx^i dx^j,
\end{equation}
where $\{g_{ij}(x)\}$ is an index notation of the matrix $G(x)$ and $dx = (dx^1, \ldots, dx^m)$.

\subsection{Immersion and Embedding}
\label{sec:iande}
\begin{figure}[!t]
    \centering
    \includegraphics[width=\linewidth]{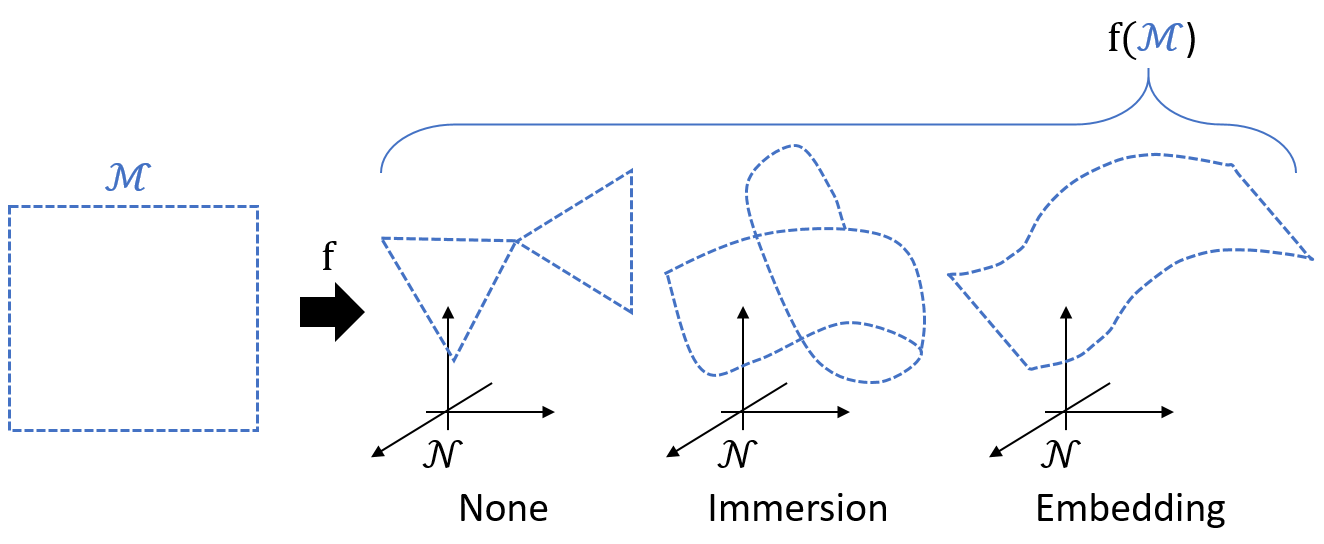}
    \caption{An illustration of immersion and embedding between two manifolds ${\cal M}$ and ${\cal N}$.}
    \label{fig:gp2}
    \vspace{-5pt}
\end{figure}
Consider two differentiable manifolds, an $m$-dimensional manifold ${\cal M}$ and $n$-dimensional manifold ${\cal N}$, with their respective coordinates $x\in\mathbb{R}^{m}$ and $y \in \mathbb{R}^{n}$. A differentiable mapping ${\rm f}:{\cal M} \to {\cal N}$ is an {\it immersion} if its differential is an injective function at every point in ${\cal M}$. Representing the mapping in local coordinates as $f:\mathbb{R}^{m} \to \mathbb{R}^{n}$, equivalently, $f$ is an immersion if its Jacobian matrix
\begin{equation}
    J(x):=\frac{\partial f}{\partial x}(x) \in \mathbb{R}^{n \times m}
\end{equation}
has constant rank equal to $\text{dim}({\cal M})=m$ at every point. Intuitively, for $f$ to be an immersion, the output manifold dimension $n$ must be greater or equal to the dimension of ${\cal M}$. 
A smooth {\it embedding} is an injective immersion ${\rm f}:{\cal M} \to {\cal N}$ such that ${\cal M}$ is diffeomorphic to its image ${\rm f}({\cal M}) \subset {\cal N}$\footnote{A manifold ${\cal A}$ is diffeomorphic to another manifold ${\cal B}$ if there exists a differentiable map between ${\cal A}$ and ${\cal B}$ such that its inverse exists and is differentiable as well.}. Specifically, when the domain manifold is compact, a smooth embedding is equivalent to an injective immersion. Then, ${\rm f}({\cal M})$ is called an {\it embedded manifold} in ${\cal N}$; see Fig.~\ref{fig:gp2}. 


\subsection{Riemannian Geometry of Parametric Curves}
\label{sec:rgpc}

\begin{figure}[!t]
    \centering
    \includegraphics[width=\linewidth]{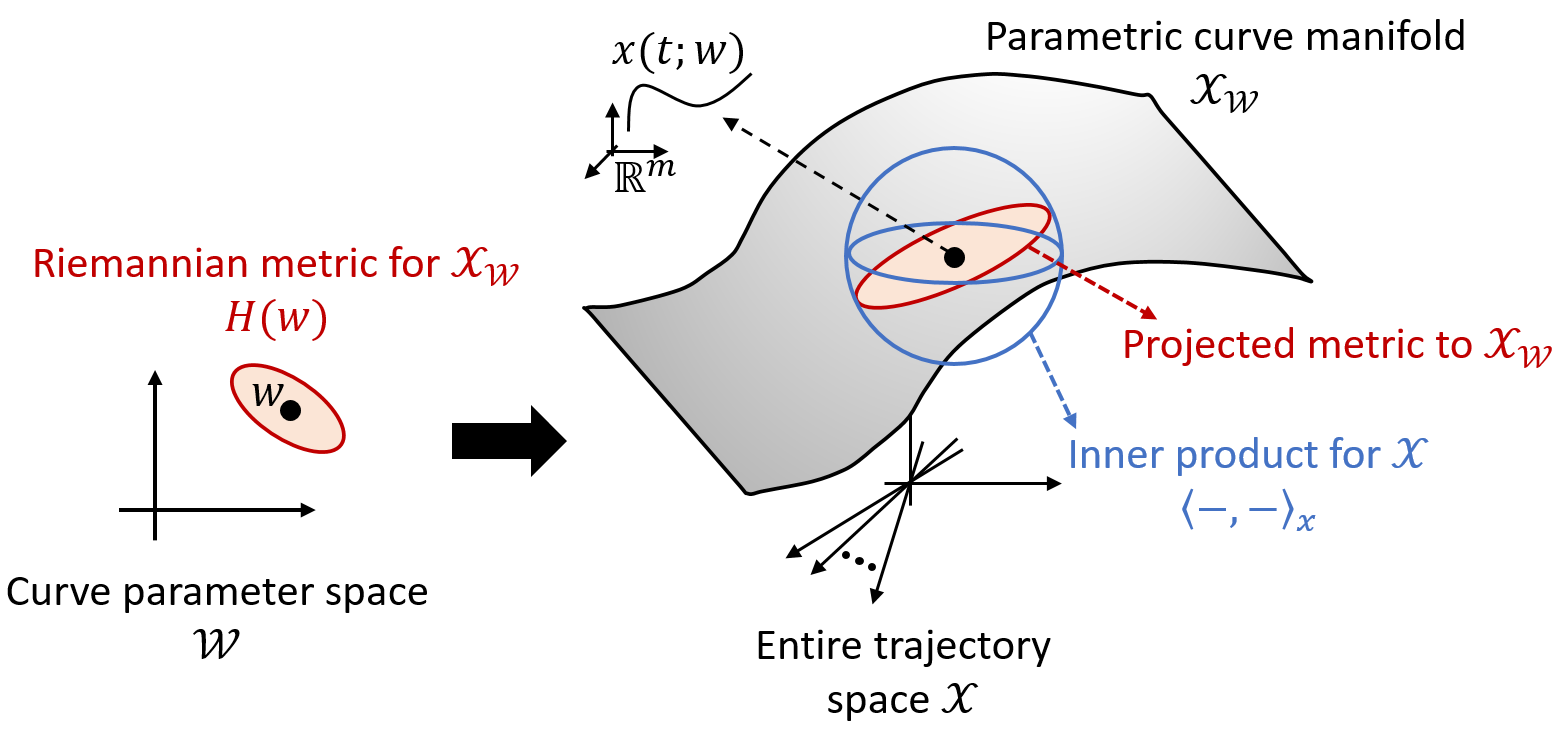}
    \caption{An illustration of Riemannian geometry of the parametric curve manifold ${\cal X}_{\cal W}$.}
    \label{fig:gp3}
    \vspace{-5pt}
\end{figure}

In this paper, we are particularly interested in Riemannian geometry of manifold of parametric curves. This section introduces how to define a Riemannian metric for the parametric curve manifold; see Fig.~\ref{fig:gp3}.

Let ${\cal M}$ be an $m$-dimensional Riemannian manifold with its local coordinates $x \in \mathbb{R}^{m}$ and Riemannian metric $G(x)$. Consider a smooth curve in ${\cal M}$ expressed as $x:[0, T] \to \mathbb{R}^{m}$ in coordinates, where its velocity norm is defined as $\|\dot{x}\|:=\sqrt{\dot{x}^T G(x) \dot{x}}$. 
The space of all smooth curves is considered an infinite-dimensional function space ${\cal X}$ with an inner product defined as follows: $\langle v, w \rangle_x:= \int_{0}^{T} v(t)^T G(x(t)) w(t) \ dt$ for two square-integrable functions $v,w:[0,T] \to \mathbb{R}^{m}$ (i.e., ${\cal X}$ is a Hilbert space). 

Of particular relavance to this paper is a parametric curve $x(t; w)$ where $w \in {\cal W}$ denotes the parameter of the curve and ${\cal W} \subset \mathbb{R}^{n}$. Consider the set of all parametric curves ${\cal X}_{{\cal W}} := \{x(\cdot;w) \in {\cal X} \ | \ w \in {\cal W}\}$. This space is a $n$-dimensional smooth manifold under the following conditions:
\begin{prop}
\label{prop:manifold}
    Suppose a curve $x(t;w)$ is smooth in both $t$ and $w$ and $x(t;\cdot): {\cal W} \to {\cal X}$ is injective, i.e., if $x(t;w_1)=x(t;w_2)$ for all $t \in [0, T]$, then $w_1 = w_2$. Let $w=(w^1,\ldots, w^n)$ and $v=(v^1,\ldots, v^n)\in \mathbb{R}^{n}$, if 
    \begin{equation}
    \label{eq:immersion}
        \sum_{i=1}^{D} \frac{\partial x(t;w)}{\partial w^{i}} v^{i} = 0 \:\:\: \implies \:\:\: v=0
    \end{equation}
    for all $w \in {\cal W}$ and ${\cal W}$ is compact, then ${\cal X}_{{\cal W}}$ is an $n$-dimensional smooth manifold. 
\end{prop}

\begin{proof}
    A smooth map $x(t; \cdot): {\cal W} \to {\cal X}$ is an injective immersion by (\ref{eq:immersion}). Since ${\cal W}$ is compact, the mapping is an embedding (i.e., ${\cal X}_{{\cal W}}$ is an embedded manifold in ${\cal X}$). 
\end{proof}

Given a parametric curve manifold ${\cal X}_{\cal W}$ embedded in ${\cal X}$, the inner product $\langle \cdot, \cdot \rangle_x$ in ${\cal X}$ can be naturally projected into ${\cal X}_{{\cal W}}$. This leads to -- by treating ${\cal W}$ as a local coordinate space for ${\cal X}_{{\cal W}}$ -- the Riemannian metric in ${\cal X}_{{\cal W}}$ expressed in the parameter space ${\cal W}$, denoted by $H(w)=\{h_{ij}(w)\}$. Specifically, the squared length of an infinitesimal displacement $dw \in \mathbb{R}^{n}$ is
\begin{align}
    ds^2 
    &= \sum_{i,j} h_{ij}(w) dw^i dw^j \nonumber \\ 
    &= \int_{0}^T \langle \sum_{i} \frac{\partial x(t;w)}{\partial w^i} dw^i,  \sum_{j} \frac{\partial x(t;w)}{\partial w^j} dw^j \rangle_x \ dt\nonumber \\
    &=\sum_{i,j} \Big( \int_{0}^T \frac{\partial x(t; w)^T}{\partial w^i} G(x(t;w))\frac{\partial x(t; w)}{\partial w^j} \ dt \Big) dw^i dw^j.
\end{align}
Therefore, 
\begin{equation}
\label{eq:rm}
    H(w) = \int_{0}^T \big( \frac{\partial x(t; w)}{\partial w}\big)^T G(x(t;w))\frac{\partial x(t; w)}{\partial w} \ dt,
\end{equation}
where $\frac{\partial x(t; w)}{\partial w} \in \mathbb{R}^{m\times n}$. 
This method of metric construction follows the standard procedure in differential geometry. A similar procedure can be found in the construction of Fisher information Riemannian metrics in statistical manifolds~\cite{amari2016information,lee2022statistical}.

\subsection{Isometry and Coordinate-Invariant Distortion Measure}
\begin{figure}[!t]
    \centering
    \includegraphics[width=0.85\linewidth]{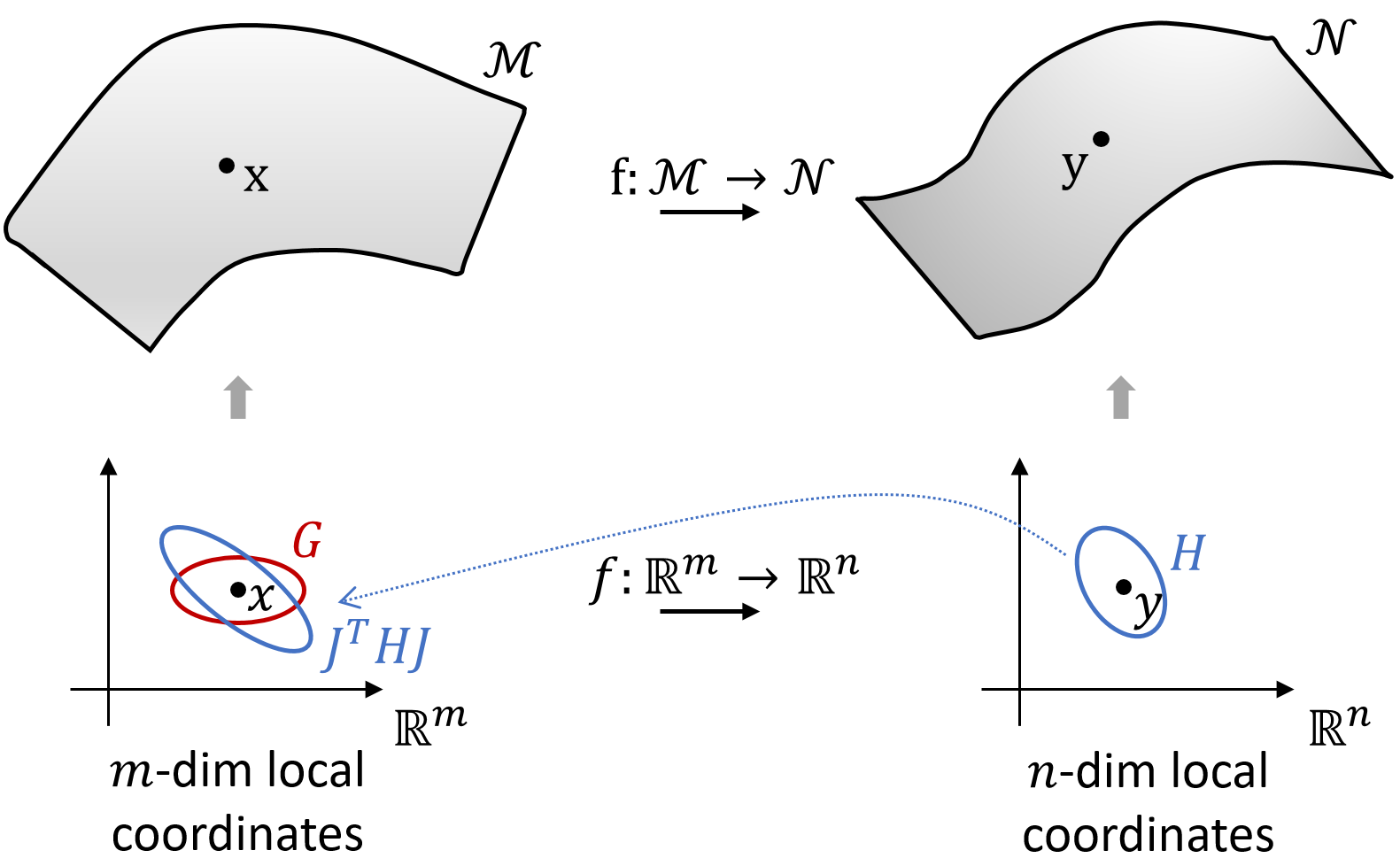}
    \caption{An illustration of a mapping between two Riemannian manifolds. If $G=J^THJ$, i.e., the red and blue ellipses coincide, at all points in $\mathbb{R}^{m}$, then the mapping $f$ is a global  isometry.}
    \label{fig:gp4}
    \vspace{-5pt}
\end{figure}

Consider two Riemannian manifolds, an $m$-dimensional manifold ${\cal M}$ with local coordinates $x\in\mathbb{R}^{m}$ and metric $G(x)$ and an $n$-dimensional manifold ${\cal N}$ with local coordinates $y\in\mathbb{R}^{n}$ and metric $H(y)$. And let $f:\mathbb{R}^{m} \to \mathbb{R}^{n}$ be a differentiable mapping between two manifolds expressed in local coordinates. 

We call $f$ an {\it isometry} if it preserves geometric structures between two spaces, i.e., preserves distances, angles, and volumes. Specifically, let $J(x)=\frac{\partial f}{\partial x}(x)$, if 
\begin{equation}
    G(x) = J(x)^T H(f(x)) J(x)
\end{equation}
at $x$, then $f$ is called a local isometry at $x$ (to see why, compare $dx^TG(x)dx$ and $dy^T H(y)dy$ where $dy=J(x)dx$). If this condition is satisfied at all points in ${\cal M}$, then $f$ is a (global) isometry; see Fig.~\ref{fig:gp4}. 

Sometimes, it is too stringent to find an isometry between two spaces and better to ignore the scale of distances~\cite{yonghyeon2022regularized}. A mapping $f$ that satisfies the relaxed condition
\begin{equation}
\label{eq:lsic}
    G(x) = cJ(x)^T H(f(x)) J(x)
\end{equation}
for all $x$ and for some positive scalar $c$ is called a {\it scaled isometry}. It preserves, angles and scaled distances.

There is a family of coordinate-invariant Riemannian {\it distortion measures}, each of which is a functional of a mapping $f$ that measures how far $f$ from being an isometry~\cite{jang2021riemannian}. 
Let $\lambda_i$ be eigenvalues of $J^T H J G^{-1}$. One example is  
\begin{equation}
    \int_{\cal M} \|\lambda_i(x) - 1\|^2 \sqrt{\det G(x)} dx.
\end{equation}
This measure is coordinate-invariant\footnote{To see why, consider a pair of coordinate transformations $\psi: x\mapsto x'$ and $\phi: y \mapsto y'$. Let $\Psi = \frac{\partial \psi}{\partial x}$ and $\Phi=\frac{\partial \phi}{\partial y}$, then the Riemannian metrics transform via $G\mapsto \Psi^{-T} G \Psi^{-1}$ and $H\mapsto \Phi^{-T}H\Phi^{-1}$, while the Jacobian transforms via $J\mapsto \Phi J \Psi^{-1}$. Therefore, $J^T H J G^{-1} \mapsto \Psi^{-T}J^T H J G^{-1}\Psi^{T}$ and the eigenvalues remain unchanged.}, and note that if $\lambda_i(x)=1$ for all $x$, then the measure is zero and $f$ is an isometry.

A family of {\it relaxed distortion measures} quantifies how far $f$ from being a scaled isometry~\cite{yonghyeon2022regularized} within the support of a positive finite measure $\nu$ in ${\cal M}$. One of them is
\begin{equation}
\label{eq:rdm}
    \int_{\cal M} \|\frac{\lambda_i(x)}{\frac{1}{\nu({\cal M})}\int_{\cal M} \frac{1}{m}\sum_{i=1}^{m} \lambda_i(x) d\nu(x)} - 1\|^2 d\nu(x).
\end{equation}
This measure is coordinate-invariant, and note that if $\lambda_i(x)=c$ for all $x$ in the support of $\nu$ for some positive scalar $c$, then the measure is zero and $f$ is a scaled isometry in $\nu$ (i.e., equation (\ref{eq:lsic}) is satisfied at all points in the support of $\nu$). 

Given a probability measure $P$ in ${\cal M}$, restricting the relaxed distortion measure~(\ref{eq:rdm}) to the support of $P$ and removing the additive constant, it is proportional to 
\begin{equation}
\label{eq:rdmff}
    {\cal R}(f; P):=\frac{\mathbb{E}_{x \sim P} [{\rm Tr} \big((J^THJG^{-1})^2\big)]}{\mathbb{E}_{x\sim P}[{\rm Tr}\big(J^THJG^{-1}\big)]^2}.
\end{equation}
Given this trace-based expression, we can employ the Hutchinson stochastic trace estimator, i.e., $\text{Tr}(A) = \mathbb{E}_{v\sim {\cal N}(0, I)}[v^T A v]$, which facilitates an efficient implementation of isometric regularization. Further details can be found in~\cite{lee2023equivariant,lee2023geometric}.
\section{Isometric Motion Manifold Primitives++}
\label{sec:immp}
In this section, we begin with the limitations of the existing MMP framework that relies on discrete-time trajectory representations and propose Motion Manifold Primitives++ (MMP++), applying the MMP framework to the continuous-time parametric curve representations. Then we adopt the isometric regularization technique~\cite{yonghyeon2022regularized} with our proposed CuveGeom Riemannian metrics, and propose Isometric Motion Manifold Primitives++ (IMMP++).  

Throughout, we will consider an $n$-dimensional Riemannian configuration manifold ${\cal Q}$ with its coordinates $q \in Q \subset \mathbb{R}^{n}$ and the metric $G(q) = \{g_{ij}(q)\}$. 

\subsection{Discrete-Time MMPs and Their Limitations} 
Motion Manifold Primitives (MMP) framework learns a continuous manifold of trajectories, providing a mapping that maps a lower-dimensional latent value to the discrete-time trajectory of a fixed length~\cite{noseworthy2020task,lee2023equivariant}. 
A trajectory data is considered as a sequence of configurations denoted by $(q_1, \ldots, q_T)$ with a fixed length $T$ and treated as an element of the high-dimensional trajectory space $Q^T:=Q\times \cdots \times Q$. 

Assuming the given demonstration trajectory dataset ${\cal D}_{\rm traj} = \{(q^{i}_1,\ldots, q^{i}_T)\}_{i=1}^{N}$ lies on some lower-dimensional manifold, an autoencoder framework is adopted to learn this manifold and its coordinates. An autoencoder consists of an encoder $g:Q^T \to {\cal Z}$ and a decoder $f: {\cal Z} \to Q^T$, where ${\cal Z} = \mathbb{R}^{m}$ is a latent coordinate space. 
These two mappings are optimized to minimize the following trajectory reconstruction loss: denoting the reconstructed trajectory by $(\hat{q}^i_1, \ldots, \hat{q}^i_T) = f(g(q^{i}_1,\ldots, q^{i}_T))$, 
\begin{equation}
\label{eq:recon_loss_discrete}
    \frac{1}{N}\frac{1}{T} \sum_{i=1}^{N} \sum_{t=1}^T d^2_Q(q^i_t, \hat{q}^i_t),
\end{equation}
where $d_Q(\cdot,\cdot)$ is some distance metric in $Q$.  

Minimizing~(\ref{eq:recon_loss_discrete}) makes ${\cal D}_{\rm traj}$ approximately lie on the image of the decoder function $f$. 
As discussed in section~\ref{sec:iande}, if $f$ is injective, the Jacobian of $f$ is $m$ everywhere, and $f$ is diffeomorphic to its image, then the image of the decoder can be considered as an $m$-dimensional manifold embedded in $Q^T$. 
Then, the mappings $g, f$ with ${\cal Z}$ take the role of the coordinate chart.
Consequently, an autoencoder can be interpreted as learning the motion manifold and its coordinates, simultaneously. 

In practice, $g$ and $f$ are approximated using deep neural networks with smooth activation functions. By setting $m$ sufficiently low, much lower than $\text{dim}(Q^T)=nT$, empirical results imply that $f$ converges to satisfy the above conditions without enforcing them. However, choosing a large $m$ may drop the rank of the Jacobian of $f$~\cite{lee2023geometric}.

However, the nature of discrete-time trajectory representation leads to multiple limitations, as listed below.
First, because the reconstructed trajectories may not be smooth, additional smoothness regularization must be added to the reconstruction loss (\ref{eq:recon_loss_discrete}), requiring weight tuning. 
Second, temporal modulation is not possible, meaning we cannot modulate the speed of the trajectory. More importantly, generating a configuration at an arbitrary specified time is not possible, which makes one of our key applications -- online iterative replanning introduced later in Algorithm \ref{alg:OIR} -- inapplicable where the time variable needs to be continuously modified. 
Lastly, given some constraints on curves (e.g., via-points), generated trajectories cannot be enforced to satisfy these constraints. We may introduce an additional regularization term to the loss (\ref{eq:recon_loss_discrete}); however, this not only requires additional weight tuning but also does not guarantee the satisfaction of the constraints. Furthermore, smooth modulation of the trajectory is not allowed, such as by changing via-points, which limits the model's adaptability. For example, if we modulate a configuration at the first time step \(q_1\), since there is no continuity between adjacent points, the other points \(q_t\) for \(t > 1\) will remain unchanged.

\subsection{Motion Manifold Primitives++}
This section proposes MMP++, an extension of the MMP framework to the continuous-time parametric curve models.
We use a phase variable $\tau \in [0,1]$, and our main subject of interest is a parametric curve model
\begin{equation}
    q: [0,1] \times {\cal W} \to Q \:\:\: \text{s.t.} \:\:\: q(\tau, w) \in Q.
\end{equation}
Then, given any monotonically increasing function with time $\tau(t)$, a timed-trajectory $q(\tau(t);w)$ can be constructed with a desired velocity profile $\frac{d}{dt}q(\tau(t); w) = \dot{\tau_t}\frac{\partial}{\partial\tau}q(\tau;w)$. This is called a temporal modulation.

Specifically, we focus on a particular class of curve models, an affine curve model, that is expressed as follows:
\begin{equation}
\label{eq:affinemodel}
    q(\tau; w) = \psi(\tau) + w \phi(\tau),
\end{equation}
where $\psi(\tau) \in \mathbb{R}^{n}$, $\phi(\tau) \in \mathbb{R}^{B}$, and $w \in \mathbb{R}^{n \times B}$. This class includes models from ProMP~\cite{paraschos2013probabilistic} and VMP~\cite{zhou2019learning}. In VMP, $\psi(\tau)$ is referred to as an elementary trajectory, and $w\phi(\tau)$ is termed a shape modulation. 

We assume that we are provided with multiple demonstration trajectories for a given task, each of which is a sequence of time-configuration pairs $((t_1, q_1), \ldots, (t_L, q_L))$. 
In the pre-processing step, we fit each demonstration trajectory to the affine curve model~(\ref{eq:affinemodel}) and find $w$.
Specifically, we set $\tau_{\text{linear}}(t) = \frac{t}{t_L}$ and consider $\Delta_{i}:=q_{i} - \psi(\tau_{\text{linear}}(t_i))$. Then, we want to find $w$ that best fits the trajectory, i.e., $\min_{w} \sum_{i=1}^{L}\|\Delta_i - w\phi(\tau_{\text{linear}}(t_i)) \|$. Assuming $L > B$, there is a closed-form solution:
\begin{equation}
    w^* = \Delta \Phi^T (\Phi \Phi^T)^{-1},
\end{equation}
where the data matrix $\Delta=(\Delta_1, \ldots, \Delta_L) \in \mathbb{R}^{n \times L}$ and basis matrix $\Phi=(\phi(\tau_{\text{linear}}(t_1)), \ldots, \phi(\tau_{\text{linear}}(t_L)))\in \mathbb{R}^{B \times L}$.

Suppose we are given curve parameters fitted to the demonstration trajectories, denoted by $\{w_1, \ldots, w_N\}$ where $w_i \in \mathbb{R}^{n \times B}$ is a curve parameter fitted to an $i$-th trajectory. 
Adopting~\cite{noseworthy2020task, lee2023equivariant}, we use an autoencoder framework to learn the manifold and its coordinates. An autoencoder consists of an encoder $g:{\cal W} \to {\cal Z}$ and a decoder $f: {\cal Z} \to {\cal W}$, where ${\cal W} = \mathbb{R}^{n\times B}$ is the curve parameter space and ${\cal Z} = \mathbb{R}^{m}$ is a latent coordinate space. 
These two mappings are optimized to minimize the following standard autoencoder reconstruction loss:
\begin{equation}
\label{eq:recon_loss}
    \frac{1}{N} \sum_{i=1}^{N} \|w_i - f(g(w_i))\|_F^2,
\end{equation}
where $\|\cdot\|_F$ is the Frobenius norm.

Minimizing~(\ref{eq:recon_loss}) makes $\{w_i\}_{i=1}^{N}$ approximately lie on the image of the decoder function $f$. 
As a result, $f$ maps the latent coordinate space ${\cal Z}$ to a manifold in the curve parameter space ${\cal W}$, 
and then the parametric curve model $q(\tau; \cdot)$ maps this manifold to a manifold of continuous-time trajectories, i.e., the motion manifold, in the trajectory space. as visualized in Fig.~\ref{fig:intro2}.



Once $f, g$ are fitted, we train a latent space distribution using the encoded data $\{g(w_i)\}_{i=1}^{N}$. To capture the multi-modality of the distribution, we use a Gaussian Mixture Model (GMM), yet any other distribution fitting methods can be used.
We call this framework {\it Motion Manifold Primitives++ (MMP++)}, where `++' is added to distinguish it from vanilla MMP that uses discrete-time trajectory representations. 

One might question why, instead of adopting a two-step approach to learn the autoencoder and latent density separately, we don't utilize the Variational Autoencoder (VAE) framework directly~\cite{kingma2013auto}. While using a VAE is a feasible approach, we have found that separating manifold learning from density learning proves more effective. This is because, in some instances, the KL divergence in VAEs -- which penalizes differences between prior and posterior distributions in the latent space -- can adversely affect the quality of reconstructions.  

Another question that may arise is why we don't fit a density model directly in the curve parameter space \({\cal W}\). The primary reason is the high-dimensionality of the curve parameter space. For example, if the configuration space dimension is 7 and the number of basis functions \(B = 30\), then the dimensionality of \({\cal W}\) is 210. Learning a density directly in this high-dimensional space is often challenging and, as shown in our later experiments, performs worse than our methods that learn densities in much lower-dimensional latent spaces. More importantly, the high dimensionality of the curve parameter space makes it unsuitable for fast adaptation, such as the online iterative replanning introduced later in Algorithm 1. This is because the high dimensionality of \({\cal W}\) (i) requires a lot of data points for accurate density estimation\footnote{Any consistent estimator for $p$-times differentiable $d$-dimensional density functions converges at a rate of at most $n^{-\frac{p}{2p+d}}$ where \(n\) is the number of samples~\cite{stone1980optimal}. Therefore, when \(d\) is large, the rate is very slow.}, which is not the case in our situation, and (ii) makes the optimization insufficiently fast.

\subsection{Isometric Regularization} 
The MMP++ often produces a geometrically distorted latent coordinate space ${\cal Z}$. 
Adopting~\cite{yonghyeon2022regularized}, we would like to minimize the distortion between ${\cal Z}$ and $f({\cal Z}) \subset {\cal W}$ by adding the relaxed distortion measure~(\ref{eq:rdmff}) of the decoder mapping $f:{\cal Z} \to {\cal W}$ to the reconstruction loss function. 

We consider the latent space ${\cal Z}$ as a Riemannian manifold assigned with the identity metric $I \in \mathbb{R}^{m \times m}$, i.e., an $m$-dimensional Euclidean space. To apply the isometric regularization~\cite{yonghyeon2022regularized}, we should be able to interpret the output space ${\cal W}$ as a local coordinate space for the embedded manifold ${\cal X}_{\cal W} = \{\psi(\tau) + w\phi(\tau) \in {\cal X} \: | \: w \in {\cal W}\}$ (see section~\ref{sec:rgpc}). 
The space ${\cal X}_{\cal W}$ is an $nB$-dimensional smooth manifold, if $\phi_1(\tau), \ldots, \phi_B(\tau)$ are linearly independent:
\begin{prop}
    Suppose $\phi_1(\tau), \ldots, \phi_B(\tau)$ are linearly independent, i.e., if $a_1 \phi_1(\tau) + a_2 \phi_2(\tau) + \cdots + a_d \phi_d(\tau) = 0$ for all $\tau\in [0,1]$, then $(a_1, \ldots, a_d) = 0$. Then, the affine curve model (\ref{eq:affinemodel}) satisfies the injective immersion condition in Proposition~\ref{prop:manifold}.
\end{prop}
\begin{proof}
    Suppose $\psi(\tau) + w_1 \phi(\tau) = \psi(\tau) + w_2\phi(\tau)$ for all $\tau$, which implies that $(w_1 - w_2)\phi(\tau)=\sum_{j=1}^{B} (w_1-w_2)^{ij} \phi_j(\tau)=0$ for all $\tau$ and $i$. By the linearity, $w_1 = w_2$; the injectivity is proved. Now, suppose $\sum_{i,j} \frac{\partial q(\tau;w)}{\partial w^{ij}} v^{ij}=0$, which implies that $\sum_{j} v^{ij} \phi_j(\tau)=0$ for all $\tau$ and $i$. Similarly, by the linearity, $v=0$; thus the mapping $w \mapsto w \phi(\tau)$ is an immersion.  
\end{proof}

In existing movement primitives~\cite{paraschos2013probabilistic, zhou2019learning}, one of the standard methods for constructing $\phi(\tau)$ involves normalizing scalar-valued functions $b_i(\tau), i=1,\ldots,B$: $\phi_i(\tau) = b_i(\tau)/\sum_{j=1}^{B} b_j(\tau)$. With this construction, if $\{b_i\}$ is linearly independent and $\sum_j b_j(\tau) > 0$ for all $\tau$, then $\{\phi_i\}$ is linearly independent as well. We can construct such $\{b_i\}$ with the following proposition:
\begin{prop} (Corollary of Proposition 4.3. in \cite{paulsen2016introduction})
\label{prop:prop3}
Let $K : \mathbb{R} \times \mathbb{R} \to \mathbb{R}$ be a positive function and $\{c_1,c_2, \ldots, c_B\}$ be a finite set of mutually distinct points. Define $b_i(\tau) = K(\tau, c_i)$.
Then $\{b_i\}$ is linearly independent if and only if the matrix $(K(c_i, c_j ))_{i,j=1,...,B}$ is positive definite.
\end{prop}

Consider the most standard choice of $b_i(\tau)$ for stroke-based movements, the Gaussian basis functions $b_i^G(\tau):=\exp(-\frac{(\tau - c_i)^2}{2h})$ where $h$ defines the width of basis and $c_i$ the center for the $i$-th basis. According to Proposition~\ref{prop:prop3}, if $\{c_1,c_2, \ldots, c_B\}$ is mutually distinct, then $\{b_i^G\}$ is linearly independent, because Gaussian kernel is positive definite.

For a smooth manifold ${\cal X}_{\cal W}$, we can now define a Riemannian metric expressed in coordinates $w\in{\cal W}$ using equation (\ref{eq:rm}). 
Since our parameter $w \in \mathbb{R}^{n \times B}$ is a matrix and has two indices $\{w^{ij}\}$, the Riemannian metric has four indices $h_{ijkl}(w)$ such that 
\begin{equation}
    ds^2 = \sum_{i,k=1}^{n} \sum_{j,l=1}^{B} h_{ijkl} (w) dw^{ij} dw^{kl}
\end{equation}
for $dw \in \mathbb{R}^{n \times B}$. Accordingly, we define a CurveGeom Riemannian metric in ${\cal W}$ as follows:
\begin{definition}
A {\bf CurveGeom Riemannian metric} for ${\cal X}_{\cal W}$ expressed in ${\cal W}$
is 
\label{eq:cgrm}
    \begin{equation}
        h_{ijkl}(w) = \int_{0}^1 \frac{\partial q(\tau; w)^T}{\partial w^{ij}} G(q(\tau;w)) \frac{\partial q(\tau; w)}{\partial w^{kl}} \ d \tau
    \end{equation}
    for $i,k=1,\ldots,n$ and $j,l=1,\ldots,B$.
\end{definition}

Given an affine curve model, the metric further simplifies to the following expression:
\begin{prop}
    Suppose $q(\tau; w) = \psi(\tau) + w\phi(\tau)$, then the CurveGeom Riemannian metric is 
    \begin{equation}
    \label{eq:hijkl}
        h_{ijkl}(w) = \int_{0}^{1} \phi_{j}(\tau) g_{ik}(q(\tau; w)) \phi_{l}(\tau) \: d\tau,
    \end{equation}
    for $i,k=1,\ldots,n$ and $j,l=1,\ldots,B$.  
\end{prop}
\begin{proof}
    Let us denote by $q = (q^1, \ldots, q^n)$. Then, $\frac{\partial q^a(\tau; w)}{\partial w^{ij}} = \frac{\partial }{\partial w^{ij}} (\sum_{ab}w^{ab}\phi_{b}) = \sum_{b}\delta_i^a \delta_j^b \phi_b = \delta_i^a \phi_j$. Therefore, the metric is $h_{ijkl} = \int \sum_{a,b} g_{ab} \delta_i^a \phi_j  \delta_k^b \phi_l \ d\tau = \int g_{ik} \phi_j \phi_l \ d\tau$.
\end{proof}

Now, we can compute the relaxed distortion measure of the decoder mapping $f:{\cal Z} \to {\cal W}$ using equation (\ref{eq:rdmff}). 
Since the metric for ${\cal Z}$ is the identity, we only need to compute $J^T H J$. 
Unlike the case in equation (\ref{eq:rdmff}), the Jacobian of our decoder $f=(f^{ij})_{i=1,\ldots,n,j=1,\ldots,B}$ is not a matrix, but has three indices $\frac{\partial f^{ij}}{\partial z^a}$ where $a=1,\ldots,m$. Therefore, instead of $J^T H J$, we can write it as follows:
\begin{equation}
    \bar{h}_{ab}(z) = \sum_{i,k=1}^{n} \sum_{j,l=1}^{B} \big(\frac{\partial f^{ij}}{\partial z}(z)\big)^T h_{ijkl}(f(z)) \frac{\partial f^{kl}}{\partial z}(z),  
\end{equation}
where $\frac{\partial f^{ij}}{\partial z} (z)\in \mathbb{R}^{1 \times m}$. We let $\{\bar{h}_{ab}(z)\}$ be an index notation of an $m \times m$ matrix $\bar{H}(z)$.
If $\bar{H}(z) = cI$ for some positive scalar $c$ for all $z$, then $f$ is a scaled isometry. 

We consider a latent space probability measure $P$ and finally define the relaxed distortion measure as 
\begin{equation}
    {\cal R}(f; P):=\frac{\mathbb{E}_{z \sim P} [{\rm Tr} \big(\bar{H}(z)^2\big)]}{\mathbb{E}_{z\sim P}[{\rm Tr}\big(\bar{H}(z)\big)]^2}.
\end{equation}
Following~\cite{yonghyeon2022regularized}, sampling from $P$ is done by $\delta z_i + (1-\delta) z_j$ where $\delta$ is uniformly sampled form $[-\eta, 1+\eta]$ (we set $\eta=0.2$ throughout) and $z_i = g(w_i)$ and $z_j = g(w_j)$ with $w_i,w_j \sim \{w_i\}_{i=1}^{N}$. The final loss function is 
\begin{equation}
     \frac{1}{L} \sum_{i=1}^{L} \|w_i - f(g(w_i))\|^2 + \alpha {\cal R}(f; P),
\end{equation}
where $\alpha$ is a regularization coefficient. 
Together with the density model fitted in the latent coordinate space, we call this framework {\it Isometric Motion Manifold Primitives++ (IMMP++)}. 

As a special case, if ${\cal Q}$ is Euclidean space, i.e., $g_{ij}(w)$ is equal to the Kronecker delta $\delta_{ij}$, then the metric formula and isometric regularization term can be further simplified. Plugging it into equation~(\ref{eq:hijkl}), the metric is simplified to 
\begin{equation}
    h_{ijkl} = \delta_{ik}\int_{0}^{1} \phi_j(\tau) \phi_l(\tau) \ d\tau.
\end{equation}
We note that it does not depend on $w$; therefore, we do not need to compute it in every iteration of the gradient descent during autoencoder training. This greatly reduces the computational cost in isometric regularization. Specifically, the matrix $\bar{H}(z)$ becomes
\begin{equation}
    \bar{h}_{ab}(z) = \sum_{i=1}^{n} \big(\frac{\partial f^{i}}{\partial z}(z)\big)^T \Phi \frac{\partial f^{i}}{\partial z}(z),  
\end{equation}
where $f^{i}=(f^{i1}, \ldots, f^{iB}) \in \mathbb{R}^{B}$, $\frac{\partial f^{i}}{\partial z}(z) \in \mathbb{R}^{B \times m}$, and $\Phi=(\int_{0}^{1}\phi_j(\tau) \phi_l(\tau) \ d\tau)_{j,l=1,\ldots,B} \in \mathbb{R}^{B \times B}$ is constant.
\section{Experiments}
In section~\ref{subsec:poam} and \ref{subsec:7racfm}, we assume Euclidean configuration spaces $Q$ and focus on examples with fixed initial and final points $q_i, q_f \in Q$, therefore we use the via-point affine curve model from VMP~\cite{zhou2019learning}: 
\begin{equation}
\label{eq:acme}
    q(\tau; w) = (1-\tau)q_i + \tau q_f + w \phi(\tau),
\end{equation}
where $\phi_i(\tau) = \tau(1-\tau) \ b_i^G(\tau) / \sum_{j} b_j^G(\tau)$ (it is trivial to show the linear independence of $\phi_i$; see Proposition~\ref{prop:prop3}).
To give hard constraints of initial and final points to $q(\tau; w)$, we multiply $\tau(1-\tau)$ to the original basis term from \cite{zhou2019learning}. 

One might wonder whether it is always necessary to design parametric curve models in an environment-specific manner. We note that the parametric curve model (\ref{eq:affinemodel}), with the choice of \(\psi(\tau)=0\) and $\phi_i(\tau)=\ b_i^G(\tau) / \sum_{j} b_j^G(\tau)$, is sufficiently expressive to model any smooth, arbitrary complex trajectory with a sufficiently large \(B\). Thus, no special modeling is required if we do not enforce any constraints on the curve.
If we have desired constraints, we can, in some cases, design suitable parametric models that guarantee these constraints are satisfied. For example, if we want \(q(\tau)\) to pass through \((\tau_i, q_i)\) and \((\tau_i, \dot{q}_i)\) for \(i = 1, \ldots, N\), we can design \(\psi(\tau)\) to be the lowest-order polynomial that satisfies \((\tau_i, q_i)\) and \((\tau_i, \dot{q}_i)\) for all \(i\), and set \(\phi(\tau) = \Pi_i (\tau - \tau_i)^2 b^G_j(\tau)/\sum_j b^G_j(\tau)\).

We compare MMP++, IMMP++, and VMP~\cite{zhou2019learning}. We utilize Gaussian Mixture Models (GMMs) to fit latent density models for MMP++ and IMMP++ unless otherwise specified. While the distribution of $w$ is assumed to be Gaussian in vanilla VMP, it has limitations in capturing multi-modal distributions. To ensure a fair comparison, we also implement VMP with GMM. When sampling new trajectories from fitted distributions, we reject the samples with likelihood values lower than the threshold value $\eta$ -- where $\eta$ is set to be the minimum value among the likelihood values of the training trajectories. 

In section~\ref{subsec:SE3}, we consider the position-orientation space 
\[
    {\rm SE}(3) := \{(p, R) \: | \: p \in \mathbb{R}^{3}, R \in {\rm SO}(3)\},
\]
where ${\rm SO}(3) := \{R \in \mathbb{R}^{3\times 3}\: | \: R^TR = I, \det(R)=1\}$ is the group of rotation matrices; $p$ and $R$ denote the position and orientation, respectively. In the position space $\mathbb{R}^{3}$, we use the affine curve model $p(\tau; w)$ in (\ref{eq:acme}). In the orientation space ${\rm SO}(3)$, we use a tailored parametric curve model $R(\tau; w)$, of which details will be given later in the corresponding section. 
In comparison to the MMP with discrete-time trajectory representations in~\cite{lee2023equivariant}, we show that our MMP++ enables additional modulations of initial and final positions and orientations. 

\subsection{Planar Obstacle-Avoiding Motions}
\label{subsec:poam}
\begin{figure}[!t]
    \centering
    \includegraphics[width=\linewidth]{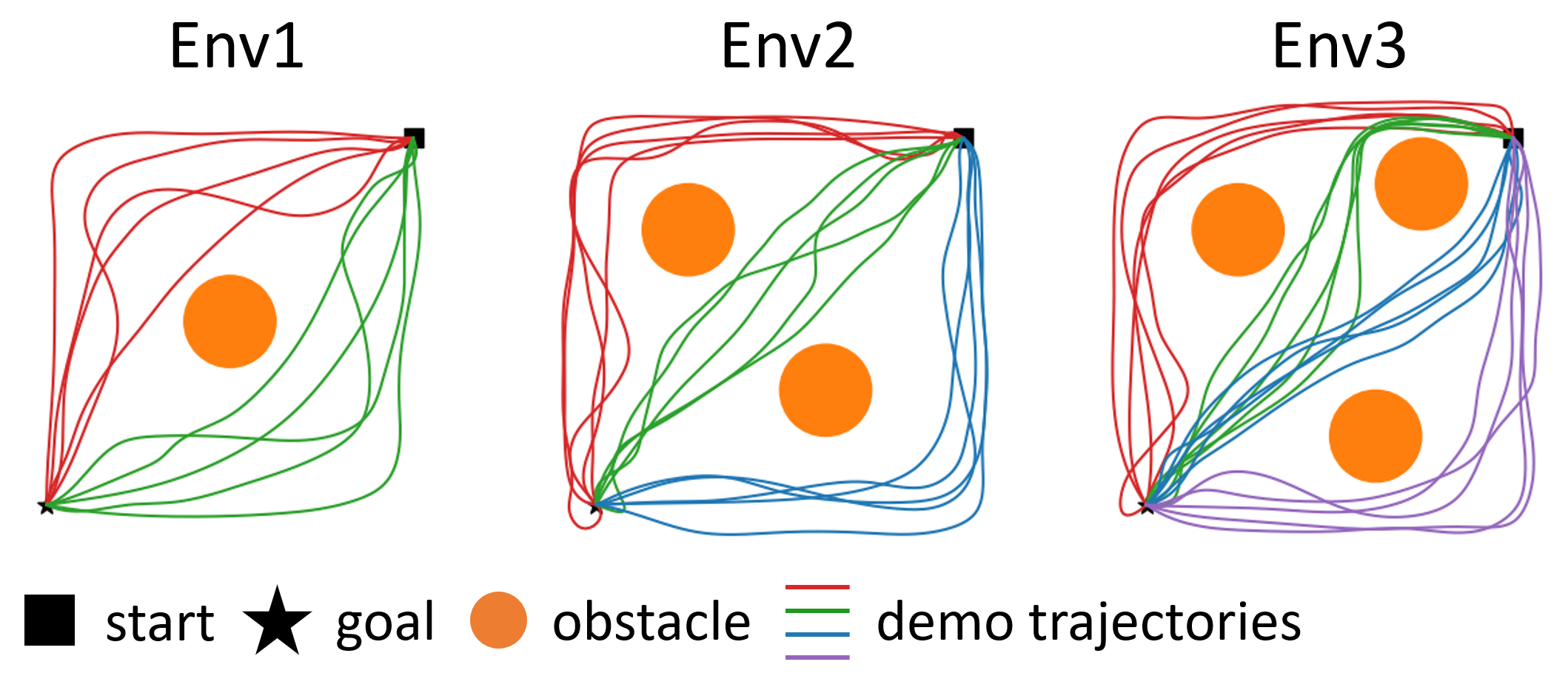}
    \caption{Three environments with training demonstration trajectories.}
    \label{fig:expsetting}
    \vspace{-5pt}
\end{figure}

We consider three different environments with different numbers of obstacles and obstacle-avoiding trajectories; see Env1, Env2, and Env3 in Fig.~\ref{fig:expsetting}. 
The number of total training trajectories for each environment setting is 10, 15, and 20, respectively. 
The number of basis $B=20$ for $\phi(\tau)$. In MMP++ and IMMP++. 
The latent space dimension is search among 2, 3, 4, and 5
The number of GMM components is set to be 2, 3, and 4 for Env1, Env2, and Env3, respectively.
A generated trajectory is considered successful if it does not collide with the obstacles. 

As shown in Table~\ref{tab:toy} and Fig.~\ref{fig:exp1}, the IMMP++ performs the best compared to the other methods. 
Sometimes, MMP++ fails significantly because GMM fits wrong clusters due to the geometric distortions in the latent coordinate spaces; see Fig.~\ref{fig:intro} for example latent spaces of MMP+ and IMMP++.

\begin{table}[!t]
    \centering
    \caption{Averages and standard deviations of the success rates with 5 times run with different random seeds; the higher the better. The best results are marked in bold.}
    \label{tab:toy}
    \begin{tabular}{l|ccc}
                         & Env1 & Env2 & Env3 \\ \hline
        VMP (Gaussian)   & 76.42 $\pm$ 1.34 & 65.76 $\pm$ 1.60 & 38.24 $\pm$ 1.18  \\
        VMP (GMM)        & 97.12 $\pm$ 0.51 & 97.52 $\pm$ 0.30 & 98.14 $\pm$ 0.24  \\
        MMP++ (ours)     & 99.48 $\pm$ 0.38 & 88.58 $\pm$ 0.60 & 97.76 $\pm$ 1.46 \\
        IMMP++ (ours)    & {\bf 100.00 $\pm$ 0.00} & {\bf 99.90 $\pm$ 0.12} & {\bf 99.22 $\pm$ 0.19}
    \end{tabular}
\end{table}

\begin{figure}[!t]
    \centering
    \includegraphics[width=1\linewidth]{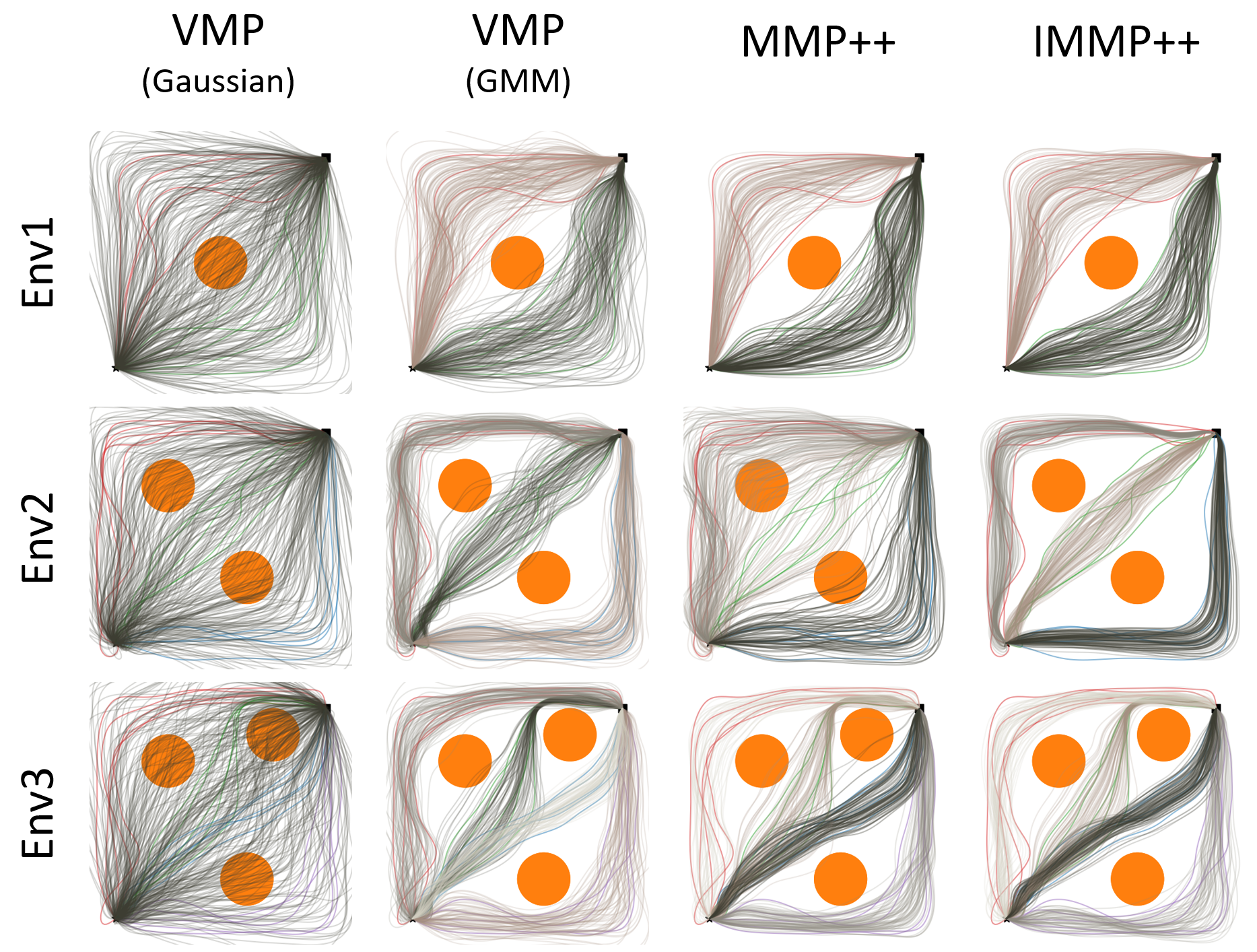}
    \caption{Trajectories generated by trained models. The GMM component numbers are 2, 3, and 4 for Env1, Env2, and Env3, respectively (samples from the same GMM component are assigned the same color).} 
    \label{fig:exp1}
    \vspace{-5pt}
\end{figure}

\subsection{7-DoF Robot Arm Collision-Free Motions}
\label{subsec:7racfm}
\begin{figure*}[!t]
    \centering
    \includegraphics[width=0.9\linewidth]{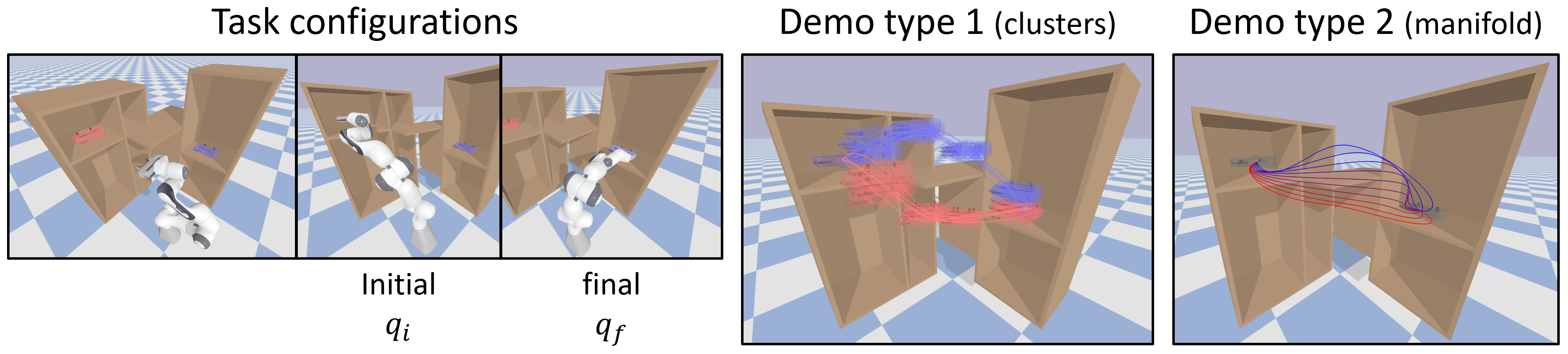}
    \caption{{\it Left}: A 7-DoF robot arm needs to move from the joint configuration $q_i$ to $q_f$ without colliding to the environment. {\it Right}: Two types of demonstration trajectories, where 7-dimensional joint space trajectories are given as demonstration data (only the end-effector's SE(3) or $\mathbb{R}^{3}$ trajectories are visualized). In demo type 1, two clusters of trajectories are given, and, in demo type 2, a one-dimensional manifold of trajectories is provided.} 
    \label{fig:robot_exp}
    \vspace{-5pt}
\end{figure*}

In this section, we consider collision-free point-to-point motions of a 7-DoF Franka Emika Panda robot arm in 
an environment shown in Fig.~\ref{fig:robot_exp} ({\it Left}). Two types of demonstration trajectories are given, each with 10 joint space trajectories -- in Fig.~\ref{fig:robot_exp} ({\it Right}), SE(3) and $\mathbb{R}^{3}$ forward kinematics results of them are visualized for simplicity --, where the initial and final configurations are identical as $q_i$ and $q_f$ in Fig.~\ref{fig:robot_exp} ({\it Left}). 
In the demo type 1, there are two clusters of trajectories, whereas in the demo type 2 a set of trajectories forms an 1-dimensional manifold, a trajectory manifold embedded in the trajectory space. The number of basis $B=20$ for $\phi(\tau)$. In MMP++ and IMMP++, the latent space dimensions are 2. The number of GMM components is 2. Exceptionally, for MMP++ and IMMP++ applied to the demo type 2, we use the Kernel Density Estimator (KDE) instead of the GMM (we will provide the reason later).

Table~\ref{tab:robot} shows the averages and standard deviations of the success rates. Overall, MMP++ and IMMP++ show the best performances.
In particular, VMPs show far inferior results than our manifold-based methods given a manifold of demonstrations in the demo type 2. This is because neither Gaussian nor GMM models are suitable for capturing the continuous manifold structure of the density's support.
Fig.~\ref{fig:robot_results} shows the two-dimensional latent spaces of the manifold-based methods. 
For demo type 1, the distance between clusters in the latent space of IMMP+ is greater than that in the latent space of MMP++. 
Even though MMP++ successfully captures correct clusters in this particular dataset, for more complex datasets, it is more likely to fail in capturing correct clustering structures.

For demo type 2, we note that two-dimensional encoded latent points form an one-dimensional manifold; 
it should be considered as a single connected manifold component. 
Given this manifold support, we found that it is sub-optimal to use the GMM for fitting a distribution. 
Hence, we use a non-parametric method, KDE, that can more accurately estimate our latent space density function:
\begin{equation}
\label{eq:KDE_density}
    p(z) = \frac{1}{N}\sum_{i=1}^{N}\frac{1}{2\pi|H_i|^{1/2}} \exp(- \frac{(z-z_i)^T H_i^{-1} (z-z_i)}{2}),
\end{equation}
where $\{z_i\}_{i=1}^{N}$ is the set of encoded latent points and $H_i,i=1,\ldots,N$ are positive-definite matrices. 
Let $K(z_i, z_k) = \exp(-\|z_i-z_k\|^2/h)$; for this example, we construct $H_i=\Sigma_i^2$ where
\begin{equation}
    \Sigma_i = \frac{\sum_{k=1}^{N} K(z_i, z_k) (z_i - z_k) (z_i - z_k)^T}{\sum_{k=1}^{N} K(z_i, z_k)}.
\end{equation}

One may ask if using such a non-parametric method in the curve parameter space ${\cal W}$ of VMP directly can lead to a better performance than using GMM. Unfortunately, this is challenging due to the high-dimensional nature of ${\cal W}$. It is worth noting that this approach is feasible in MMP++ and IMMP++ because we utilize sufficiently low-dimensional latent spaces.

\begin{table}[!t]
    \centering
    \caption{Averages and standard deviations of the success rates with 5 times run with different random seeds; the higher the better. The best results are marked in bold.}
    \label{tab:robot}
    \begin{tabular}{c|cccc}
        Demo & \begin{tabular}{c} VMP \\ (Gaussian) \end{tabular} &  \begin{tabular}{c} VMP \\ (GMM)\end{tabular} & \begin{tabular}{c} MMP++ \\ (ours) \end{tabular} & \begin{tabular}{c} IMMP++ \\ (ours) \end{tabular} \\ \hline
         type 1 & 46.1 $\pm$ 4.37 & 97.1 $\pm$ 1.66 & 98.6 $\pm$ 0.20 & {\bf 99.5 $\pm$ 0.77} \\
         type 2 & 79.4 $\pm$ 2.94 & 83.0 $\pm$ 3.78 & {\bf 99.3 $\pm$ 0.68}  & 98.2 $\pm$ 0.68  \\
    \end{tabular}
    \vspace{-5pt}
\end{table}

\begin{figure}[!t]
    \centering
    \includegraphics[width=1\linewidth]{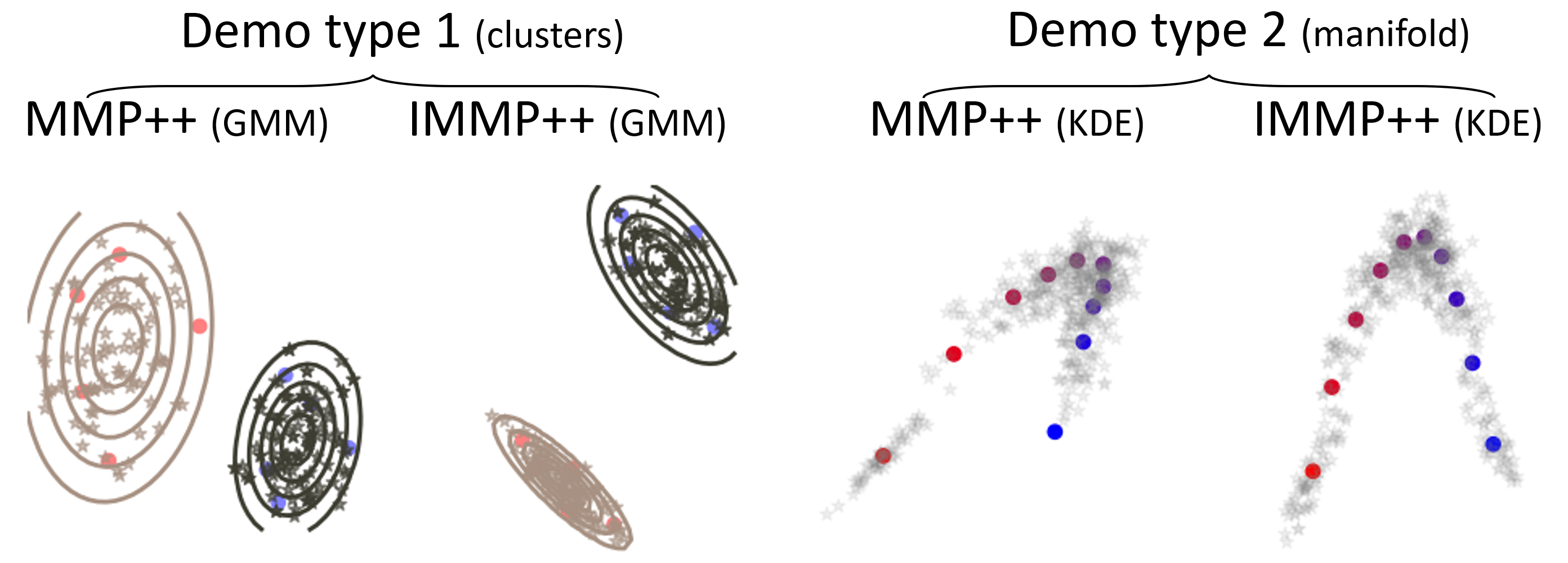}
    \caption{The two-dimensional latent spaces of our manifold-based methods (MMP++ and IMMP++) are illustrated (circular points are encoded points and star-shaped points are generated samples). Demonstration trajectories in Fig.~\ref{fig:robot_exp} ({\it Right}) are encoded into these latent spaces, with colors indicating correspondence (e.g., a red trajectory is encoded into a red point).} 
    \label{fig:robot_results}
    \vspace{-5pt}
\end{figure}

Next, we qualitatively show the latent coordinates and via-points modulation results of IMMP++ trained with the demo type 2 (manifold) in Fig.~\ref{fig:manifol_results}. Note that we have the model $q(\tau;f(z))=(1-\tau)q_i + \tau q_f + f(z)\phi(\tau)$ with the trained decoder function $w=f(z)$, where modulations in the latent coordinate value $z$ and the initial and final configurations $q_i$ and $q_f$ lead to corresponding changes in the resulting trajectory $q(\tau)$. In Fig.~\ref{fig:manifol_results} ({\it Upper}), a continuous change in the latent value $z$ leads to a smooth transition of $q(\tau)$ from an upward-moving path to a downward-moving path. As shown in Fig.~\ref{fig:manifol_results} ({\it Middle} and {\it Lower}), continuous changes in $q_i$ and $q_f$ induce smooth changes of $q(\tau)$. 

\begin{figure*}[!t]
    \centering
    \includegraphics[width=0.85\textwidth]{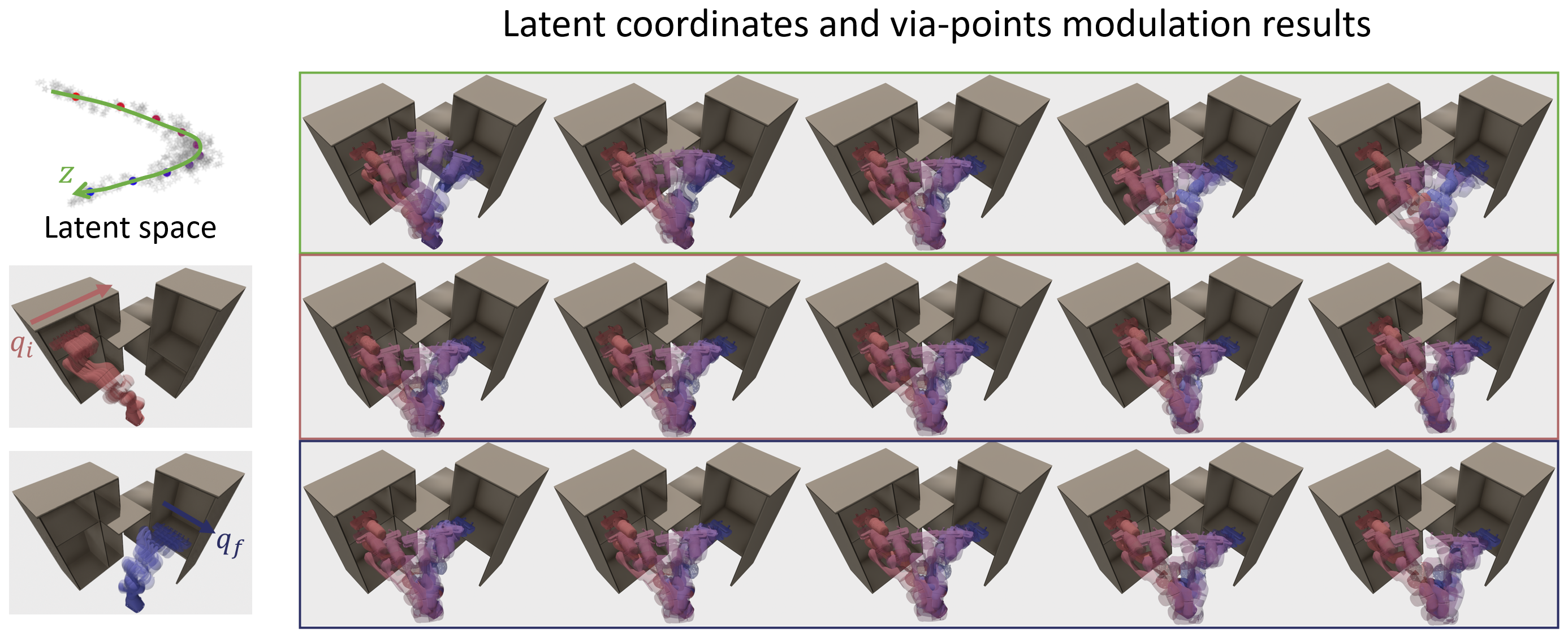}
    \caption{Modulations of latent coordinates $z$ and initial and final configurations $q_i$ and $q_f$ in IMMP++ $q(\tau;f(z))=(1-\tau)q_i + \tau q_f + f(z)\phi(\tau)$. {\it Upper}: A continuous change in $z$ results in a smooth transition of $q(\tau)$. {\it Middle}: Changes in $q_i$ induce smooth transitions of $q(\tau)$. {\it Lower}: Changes in $q_f$ induce smooth transitions of $q(\tau)$.} 
    \label{fig:manifol_results}
    \vspace{-5pt}
\end{figure*}

Lastly, we show the adaptability of our primitive models in the presence of unseen constraints and propose a novel online iterative re-planning algorithm. We will consider $(z, \tau)$ as a state. Let $T$ be a total time length. First of all, suppose we do not have additional unseen constraints. We start with a random initial $z \sim p(z)$ and $\tau=0$. Then, as the time $\delta t$ passes, $z$ remains constant and $\tau$ is updated to $\min(\tau + \delta t/T, 1)$. At each time step, we use $q(\tau; f(z))$ as a desired joint configuration and run a position controller with the control frequency $f_c$. This procedure is just a position tracking control given an initially-planned trajectory $q(\tau; f(z))$. 

Now, consider a situation where an obstacle that blocks the initially-planned trajectory appears and moves as the time flows as shown in Fig.~\ref{fig:MPC}. We assume constraints are given as inequality equations $C(q) \leq 0$ in the configuration space ${\cal Q}$, which can change dynamically in time. Suppose our current state is $(z,\tau)$. We set a time window $t_w$ and if $C(q(\bar{\tau}; f(z)) > 0$ for some $\tau < \bar{\tau} < \tau + t_w/T$, we re-plan the trajectory, i.e., update and find a new desired state $(z',\tau')$, with the re-planning frequency $f_p < f_c$. 
Given a new desired state $(z', \tau')$, we will update $z \xleftarrow{} z + k (z'-z)$ and $\tau \xleftarrow{} \tau + k (\tau'-\tau)$ with positive gain $k$ at control frequency $f_c$ for time $1/f_p$, and then decide whether we will re-plan $(z,\tau)$ again or not. Considering this update rule, in the re-planning step, we find $(z',\tau')$ by solving the following optimization problem:
\begin{align}
\label{eq:mpc_sa}
    \min_{(z', \tau')}  \: \|z - &z'\|^2  + \alpha \|\tau - \tau'\|^2 \\
    {\rm such} \: {\rm that} \:\: & {\rm (1)} \:\: C(q(\bar{\tau}'; f(z')) \leq 0 \nonumber \\ 
    & \:\:\:\:\:\:\:\:\: \tau' < \bar{\tau}' < \tau' + t_w/T \nonumber \\
    & {\rm (2)} \:\: \log p(z') \geq \epsilon \nonumber \\ 
    & {\rm (3)} \:\:  C(q(\tau_{\eta}; f(z_\eta)) \leq 0 \:\: \& \:\: \log p(z_\eta) \geq \epsilon \nonumber \\ 
    & \:\:\:\:\:\:\:\:\: z_\eta=\eta z + (1-\eta) z' \nonumber \\ 
    & \:\:\:\:\:\:\:\:\: \tau_\eta = \eta \tau + (1-\eta)\tau' \nonumber \\
    & \:\:\:\:\:\:\:\:\: \eta \in [0,1] \nonumber \\
    & {\rm (4)} \:\: \tau' \in [\tau - \delta, \tau] \nonumber, 
\end{align}

where the weight $\alpha > 0$ is set to be big enough to prioritize updating $z$. (1) enforces the planned trajectory from $(z',\tau')$ to satisfy the constraints for time $t_w$, (2) prevents $z'$ from being out-of-distribution in the latent space ($\epsilon$ is the log probability density threshold; we set it to be the minimum value among the log likelihood values of the training data), (3) guarantees the linear path connecting the current and next states from $(z,\tau)$ to $(z',\tau')$ satisfies the constraints and lies within the in-distribution region in the latent space, and (4) allows $\tau' \in [\tau-\delta, \tau]$ -- where $\delta$ is the max travel-back time -- in case no feasible $z'$ exists when $\tau' = \tau$. 
A pseudocode for our online iterative re-planning algorithm is provided in Algorithm~\ref{alg:OIR}.
\begin{algorithm}[!t]
\caption{Online Iterative Trajectory Re-planning with MMP++}
\label{alg:OIR}
\KwIn{MMP++ decoder $w=f(z)$, parametric curve $q(\tau; w)$, latent density $p(z)$, inequality constraint $C(q)\leq 0$, time window $t_w$, total time length $T$, gain $k>0$, control frequency $f_c$, re-planning frequency $f_p$, and optimization solver SOLVE (\ref{eq:mpc_sa}) }
{\bf Initialization:} $z \sim p(z)$, $\tau=0$, $c_v={\rm False}$ \\
\While{True}{
Once every $1/f_p$ seconds: \\
  \eIf{$C(q(\bar{\tau}; f(z))) > 0$ for some $\bar{\tau} \in [\tau, \min(\tau + t_w/T, 1)]$}{
    $(z_g, \tau_g) \xleftarrow[]{}$ SOLVE (\ref{eq:mpc_sa}) \\
    $c_v={\rm True}$
  }{
    $c_v={\rm False}$
  }
Once every $1/f_c$ seconds: \\
  \eIf{$c_v$}{
    $z \xleftarrow[]{} z + k(z_g - z)$ \\ 
    $\tau \xleftarrow[]{} \tau + k(\tau_g - \tau)$
  }{
    $z \xleftarrow{} z$ \\
    $\tau \xleftarrow[]{} \min(\tau + 1/(f_cT), 1)$
  }
Position control input: $q(\tau; f(z))$
}
\end{algorithm}


We can efficiently solve the optimization (\ref{eq:mpc_sa}) thanks to the low-dimensionality of the optimization variable $(z', \tau')$. 
In the example in Fig.~\ref{fig:MPC}, we use IMMP++ trained with the demo type 2 (manifold). As visualized in Fig.~\ref{fig:MPC} ({\it Upper-Right}), given a dynamically changing constraint (i.e., moving spherical obstacle), the robot iteratively re-plans its trajectory whenever it is predicted to collide with the obstacle within the time interval $t_w$. A gradient-free sampling-based optimization approach is fast enough and shows a successful online iterative re-planning result. Fig.~\ref{fig:MPC} ({\it Lower}) shows $\|z\|$ and $\tau$ as functions of time, indicating when $(z,\tau)$ is re-planned. It can be confirmed that at around $t=2$ the robot predicted collisions with the obstacle within $t_w=1$ and re-planned $(z, \tau)$ multiple times.

\begin{figure}[!t]
    \centering
    \includegraphics[width=0.9\linewidth]{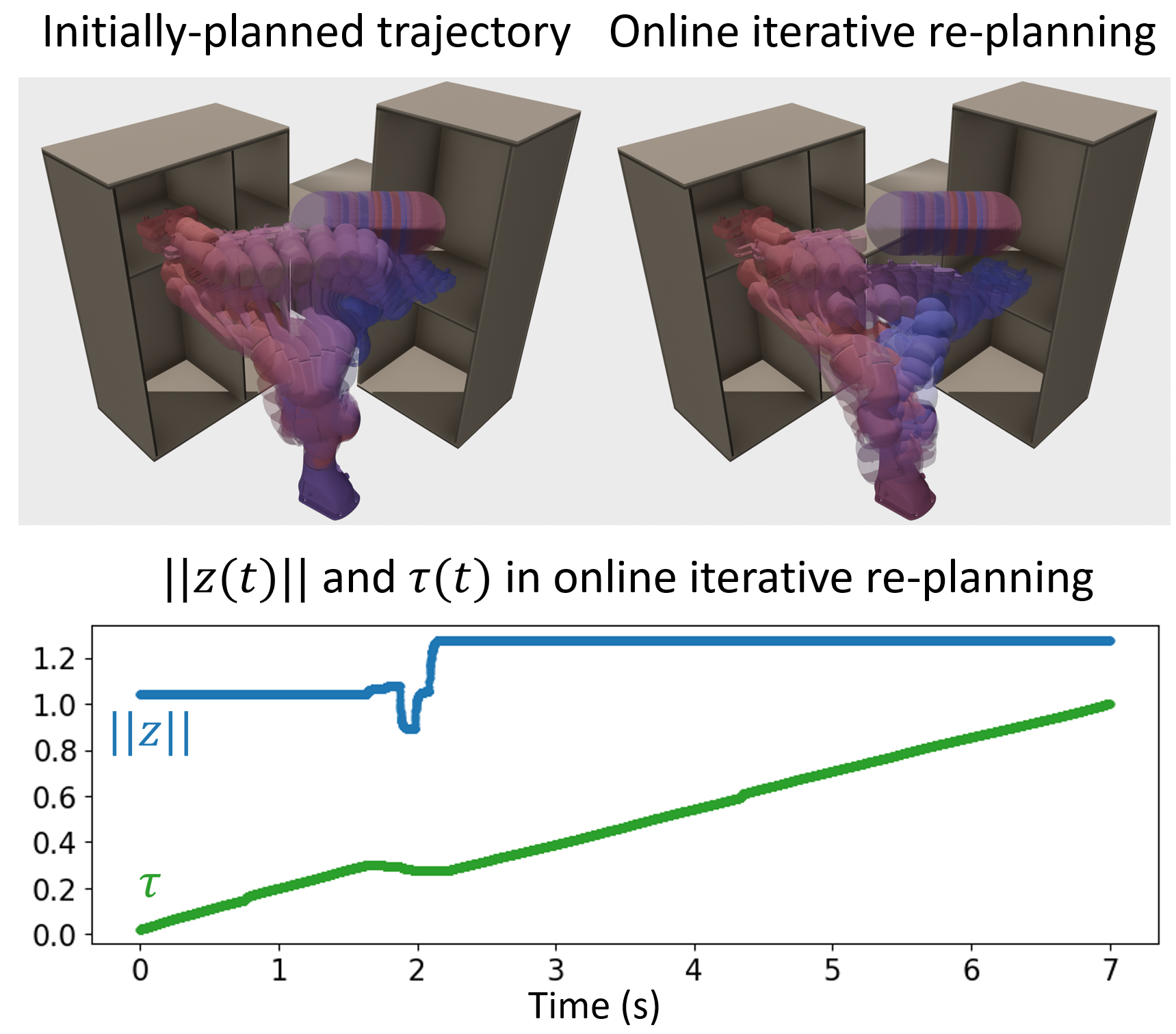}
    \caption{{\it Upper-Left:} Initially-planned trajectory is blocked by the unseen obstacle during training. {\it Upper-Right:} The robot iteratively re-plans its trajectory whenever it is predicted to collide with the obstacle. {\it Lower:} $\|z\|$ and $\tau$ as functions of time $t$ in online iterative re-planning ({\it Upper-Right}), where $T=5$, $t_w=1$, $f_c=1000$, and $f_p=10$.} 
    \label{fig:MPC}
\end{figure}

\begin{table}[!t]
    \centering
    \small
    \caption{Computation times for planning via RRT-Connect~\cite{kuffner2000rrt} and IMMP++ are measured with three times run. We consider a shelf-only environment, the shelf with one spherical obstacle, and the shelf with two spherical obstacles; see Fig.~\ref{fig:IMMP_vs_RRT}. Computations are performed using an AMD Ryzen 9 5900x 12-Core Processor and an NVIDIA GeForce RTX 3090. All code is implemented in Python.}
    \label{tab:IMMP_vs_RRT}
    \begin{tabular}{l|cc}
    & RRT-connect~\cite{kuffner2000rrt} & IMMP++ (ours) \\ \hline
    Shelf & $1.97 \sim 4.81 \  s$ & $0.006 \pm 0.001 \ s$\\ 
    Shelf + One obstacle & $60.41 \sim 134.36 \ s$ & $0.039 \pm 0.001 \ s$ \\ 
    Shelf + Two obstacles & $45.97 \sim 204.56 \ s$ & $0.077 \pm 0.002 \ s$ \\ 
    \end{tabular}
\end{table}

\begin{figure}[!t]
    \centering
    \includegraphics[width=1\linewidth]{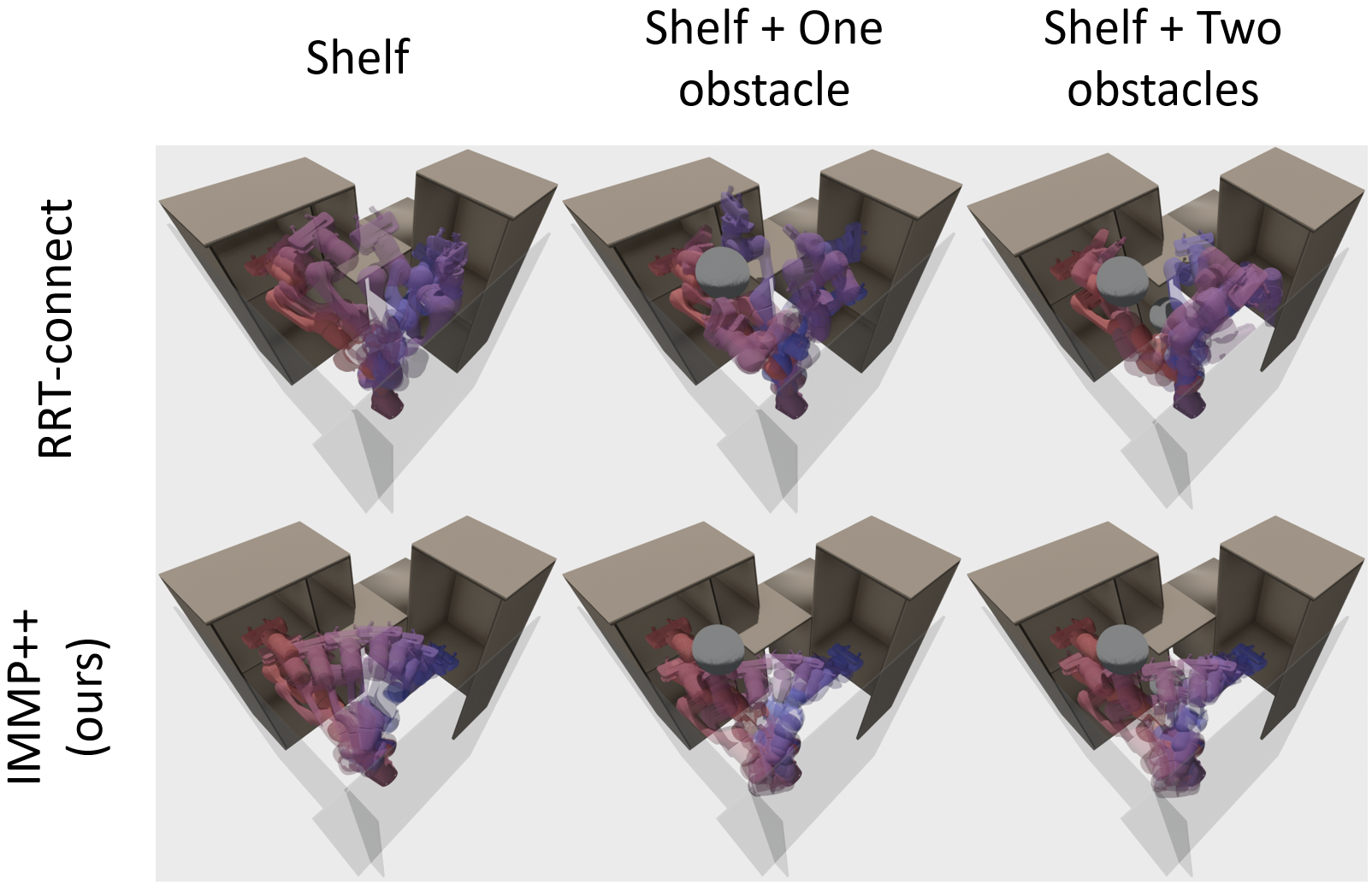}
    \caption{Comparison of collision-free paths planned by RRT-connect~\cite{kuffner2000rrt} and IMMP. Two hypothetical walls, shown as transparent gray, are considered to limit the robot's workspace to the area close to the shelf.}
    \label{fig:IMMP_vs_RRT}
    \vspace{-5pt}
\end{figure}

One can reasonably ask, if we have access to the robot kinematics and the geometry of the environment, what is the advantage of using our method compared to conventional sampling-based collision-free path planning methods, such as RRT variants? Our method offers two important advantages. First, our motion manifold primitives already contain paths that are collision-free for a part of the environment, so only the newly introduced environmental constraints need to be considered during planning. For example, in Fig.~\ref{fig:manifol_results}, the paths in the learned motion manifold do not collide with the shelf. Therefore, when planning within the motion manifold, exact knowledge of the shelf's geometry is unnecessary. While existing sampling-based planning methods require knowledge of all constraints, we only need to consider the new environmental constraints, such as obstacles shown in Fig.~\ref{fig:MPC}.

Second, a more significant advantage is that, while conventional RRTs are typically not fast enough for online dynamic planning, our method is very fast. Table~\ref{tab:IMMP_vs_RRT} shows the computation times for planning using RRT-connect~\cite{kuffner2000rrt} and our IMMP++ in the shelf-only environment and with one and two sphere obstacles, as visualized in Fig.~\ref{fig:IMMP_vs_RRT}. RRT-connect takes several seconds even in the simple shelf-only environment, and the time increases significantly as the complexity of the collision-free configuration space grows due to obstacles. In contrast, our method is relatively much faster. The trajectory sampling in the motion manifold itself takes 0.006 seconds, and the collision checking for obstacles takes only a few milliseconds. Our current implementation is in Python and is sub-optimal. We anticipate that implementing it in a more optimal language would result in even faster performance. Although the RRT algorithm could also be optimized for speed, it is unlikely to become fast enough and suitable for online dynamic planning.

However, these advantages of our algorithm do not come for free. Our method also has limitations. If the learned motion manifold primitives are too small, meaning the generated trajectories are not sufficiently diverse, a collision-free path may not exist on the motion manifold when excessive environmental constraints are introduced. The diversity of the learned motion manifold heavily depends on the diversity of the demonstration data, making the creation of good demonstration trajectories a very important issue.

\subsection{MMP++ for SE(3) Trajectory Data}
\label{subsec:SE3}

In this section, we show how to extend our MMP++ framework to SE(3) trajectory data. 
We will denote the position by $p\in\mathbb{R}^{3}$ and the orientation by $3\times 3 $ rotation matrix $R\in\mathbb{R}^{3\times 3}$. Let $\phi_i(\tau) = \tau(1-\tau) \ b_i^G(\tau) / \sum_{j} b_j^G(\tau)$ for $i=1,\ldots, B$. For the position trajectory, we use the via-point model 
\begin{equation}
p(\tau;w_p) = (1-\tau) p_i + \tau p_f + w_p \phi(\tau),    
\end{equation}
where $w_p \in \mathbb{R}^{3 \times B}$ is the position curve parameter and $p_i,p_f \in \mathbb{R}^{3}$ are initial and final points. 
For the orientation trajectory, we use the following parametric model:
\begin{equation}
\label{eq:rotation_vm}
    R(\tau; w_R) = R_i \exp ( \tau \log (R_i^T R_f))\exp ([w_R\phi(\tau)]),
\end{equation}
where $w_R \in \mathbb{R}^{3 \times B}$ is the orientation curve parameter, $R_i, R_f \in {\rm SO}(3)$ are initial and final orientations, and $\exp(\cdot)$, $\log(\cdot)$ are matrix exponential and logarithm, and 
\[
[(w^1, w^2, w^3)] = \begin{bmatrix}
    0 & -w^3 & w^2 \\ w^3 & 0 & -w^1 \\ -w^2 & w^1 & 0 
\end{bmatrix}
\] is a skew symmetric matrix~\cite{lynch2017modern}.
The orientation via-point model in (\ref{eq:rotation_vm}) is constrained to have the initial and final orientations $R_i$ and $R_f$.

\begin{figure}[!t]
    \centering
    \includegraphics[width=0.85\linewidth]{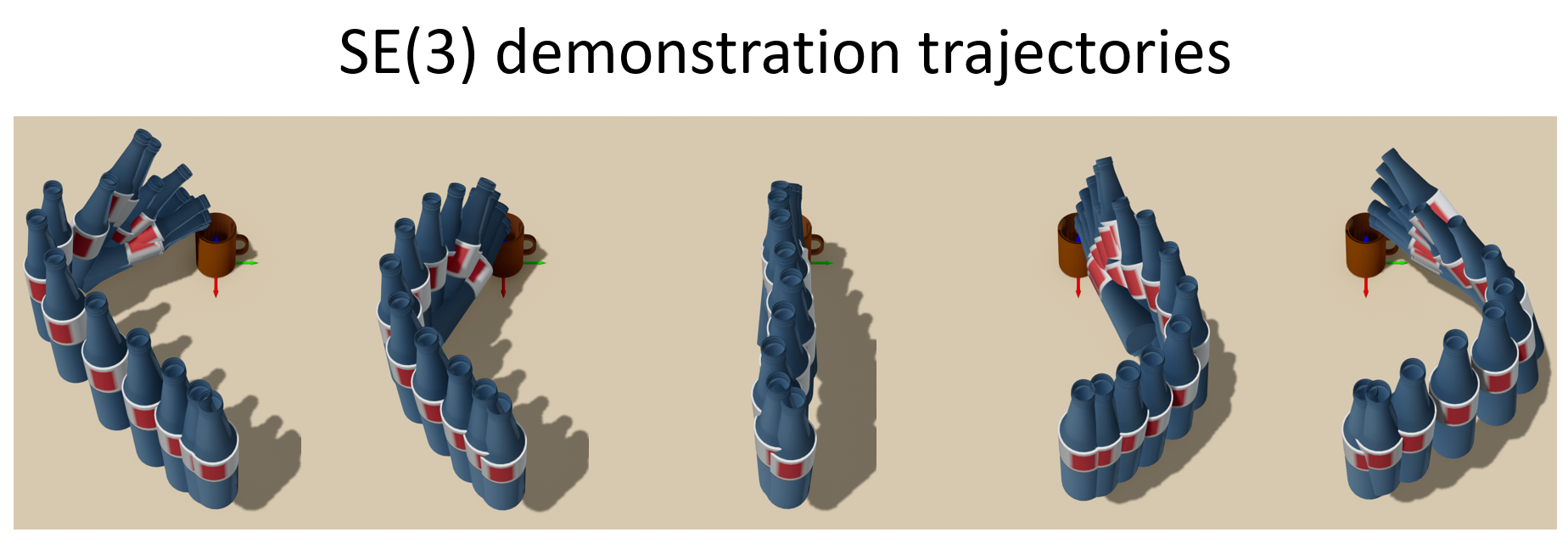}
    \caption{Water-pouring SE(3) trajectory dataset.} 
    \label{fig:wp_dataset}
    \vspace{-5pt}
\end{figure}

\begin{figure*}[!t]
    \centering
    \includegraphics[width=0.85\textwidth]{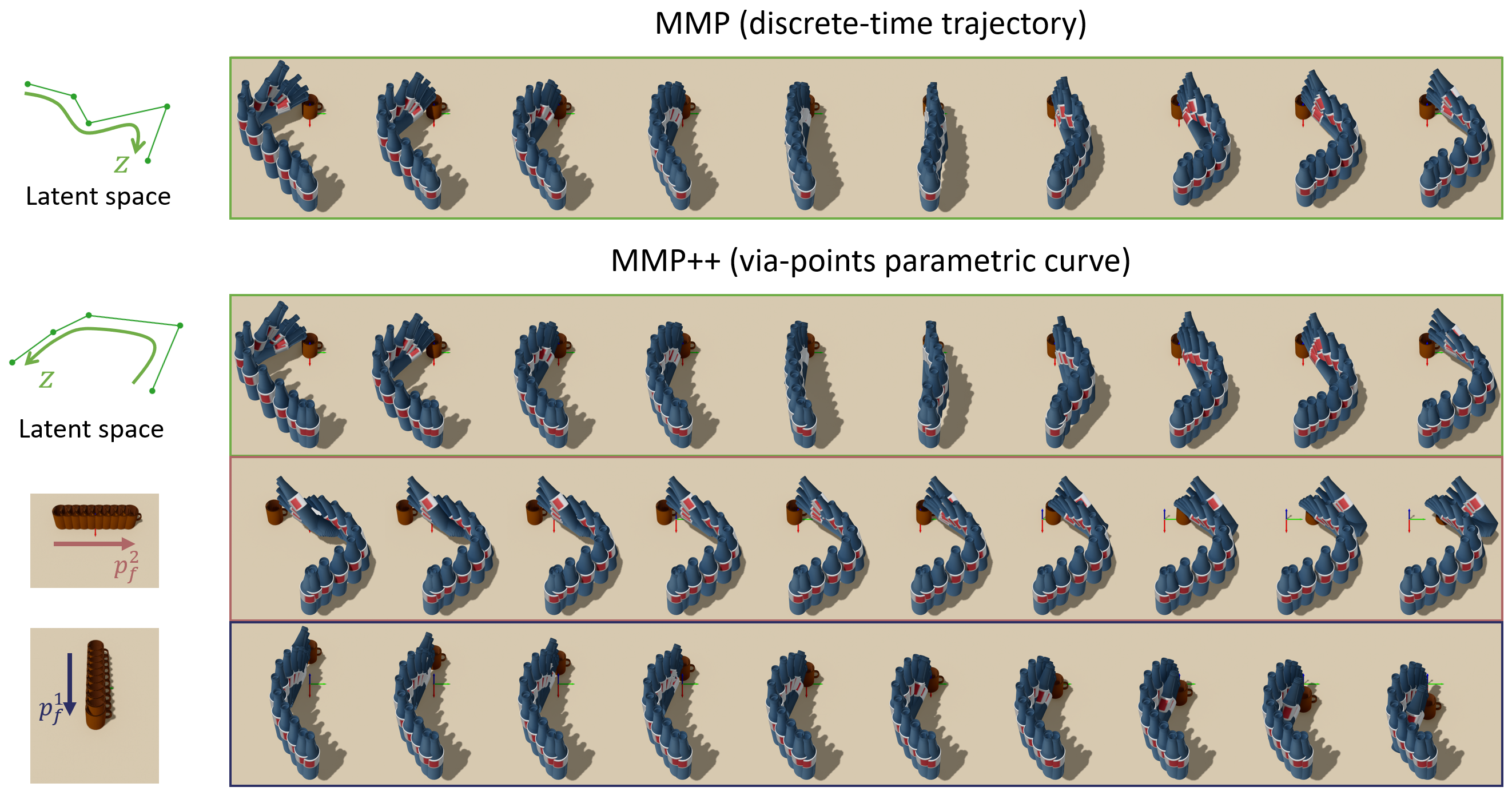}
    \caption{The modulation of MMP trained to generate discrete-time trajectories and MMP++ trained to generate curve parameters of the trajectories. We can modulate latent values for both methods and generate smooth transitions of the SE(3) trajectories. In MMP++ ({\it Middle} and {\it Lower}), given continuous changes in the cup position, the water-pouring trajectories undergo smooth transitions.} 
    \label{fig:wp_results}
    \vspace{-5pt}
\end{figure*}

\begin{figure}[!t]
    \centering
    \includegraphics[width=0.9\linewidth]{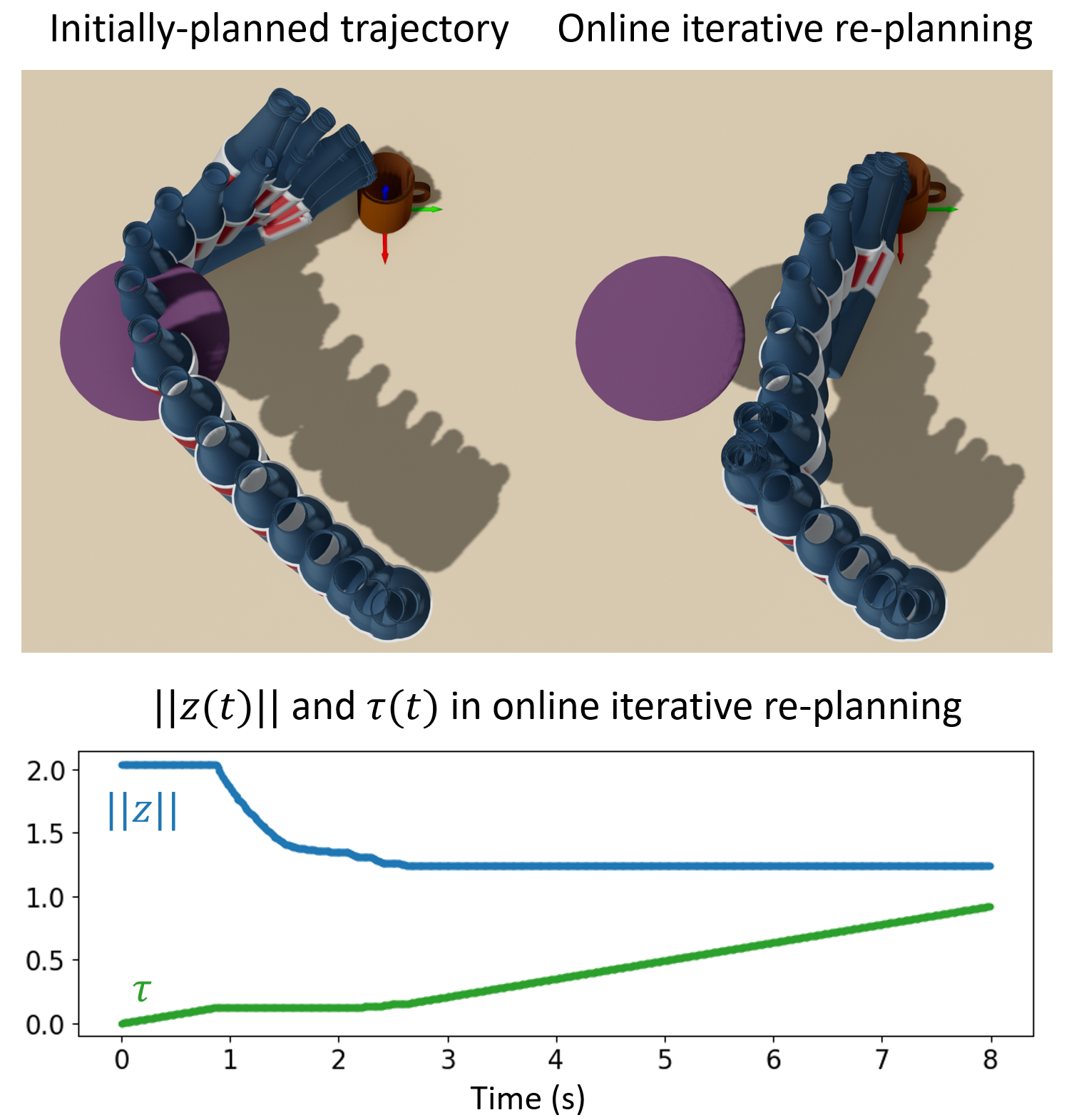}
    \caption{{\it Upper-Left:} Initially-planned trajectory is blocked by the unseen obstacle during training. {\it Upper-Right:} We iteratively re-plan the bottle's SE(3) trajectory whenever it is predicted to collide with the obstacle. {\it Lower:} $\|z\|$ and $\tau$ as functions of time $t$ in online iterative re-planning ({\it Upper-Right}), where $T=7$, $t_w=0.2$, $f_c=100$, and $f_p=10$.} 
    \label{fig:MPC2}
    \vspace{-5pt}
\end{figure}

Consider the water-pouring SE(3) trajectory dataset in Fig.~\ref{fig:wp_dataset}. Our goal is to train an MMP++ with this dataset. 
First, note that the five trajectories in the dataset have the same initial SE(3) pose $(p_i, R_i)$ but have different final SE(3) poses $(p_f, R_f)$. Therefore, in addition to $(w_p, w_R)$, the final SE(3) pose $(p_f, R_f)$ is added to the curve parameter. To emphasize this, we denote the parametric curves by $p(\tau; w_p, p_f)$ and $R(\tau; w_R, R_f)$. Thus, our autoencoder $g,f$ is trained by using $\{(w_p, w_R, p_f, R_f)_i\}_{i=1}^{5}$ where each of these parameters is fitted to the demonstration trajectory. 

To construct neural network encoder and decoder, some cares need to be taken since $R_f\in{\rm SO}(3)$ is not a vector data. For the encoder $g: (w_p, w_R, p_f, R_f) \mapsto z$, we can flatten $R_f$ to a 9-dimensional vector and treat $(w_p, w_R, p_f, R_f)$ as a $6B+12$-dimensional vector input. For the decoder $f:z \mapsto (w_p, w_R, p_f, R_f)$, the orientation output $R_f$ should satisfy the rotation matrix constraints. To enforce this condition, we let $f:z \mapsto (w_p, w_R, p_f, w_f)$ where $w_f \in \mathbb{R}^{3}$ and map $w_f \mapsto \exp([w_f]) \in {\rm SO}(3)$.

For autoencoder training, we may use the L2 reconstruction loss in the parameter space $(w_p, w_R, p_f, R_f) \in \mathbb{R}^{6B + 12}$ as in (\ref{eq:recon_loss}). 
However, we found that this loss does not suitably capture the difference between the training and reconstructed SE(3) trajectories, resulting in an autoencoder with poor reconstruction qualities. Therefore, we use the following loss function: 
let $(\hat{w}_p, \hat{w}_R, \hat{p}_f, \hat{R}_f) = f\circ g(w_p,w_R,p_f,R_f)$ and $(p_\tau, R_\tau)$ be the SE(3) demonstration trajectory -- where $\tau$ denotes the normalized time parameter (i.e., for a total demonstration time length $T$, $\tau=t/T$) --, denoting the trajectory dataset by ${\cal D}=\{(p_\tau, R_\tau)_{i}\}_{i=1}^{5}$, the new reconstruction loss is    
\begin{align}
    \sum_{(p_\tau,R_\tau) \in {\cal D}} \int_{0}^{1} & \|p_\tau - p(\tau; \hat{w}_p, \hat{p}_f)\|^2  \nonumber \\ 
    & + \beta \| \log (R_\tau^T R(\tau; \hat{w}_R, \hat{R}_f)) \|_F^2 \: d\tau,
\end{align}
where $\beta>0$ and the integration over $\tau$ is approximately computed by the discrete sum.  

Fig.~\ref{fig:wp_results} shows the modulation results of the MMP trained with discrete-time trajectories as in~\cite{lee2023equivariant} and MMP++ with our parametric curve models. 
Both methods enable modulation of the latent value $z$, producing continuous transitions of water-pouring trajectories from left to right. 
Furthermore, MMP++ enables modulation of the initial and final poses $p_i,R_i,p_f,R_f$. As shown in Fig.~\ref{fig:wp_results} (MMP++: {\it Middle} and {\it Lower}), MMP++ produces pouring trajectories that adapt to changes in the cup position.

Lastly, we apply the online iterative re-planning algorithm using the water-pouring MMP++ in the presence of unseen obstacle. We use the re-planning algorithm in Algorithm~\ref{alg:OIR}. 
Similar to the robot arm experiment with the demo type 2 (manifold), we use the KDE density model (\ref{eq:KDE_density}). 
While the initially-planned SE(3) trajectory of the bottle collides with the purple spherical obstacle as shown in Fig.~\ref{fig:MPC2} ({\it Upper-Left}), the bottle's SE(3) trajectory is re-planned online when it is predicted to collide with the obstacle and successfully avoids the collision as shown in Fig.~\ref{fig:MPC2} ({\it Upper-Right}). Fig.~\ref{fig:MPC2} ({\it Lower}) shows $\|z\|$ and $\tau$ as functions of time, indicating when $(z,\tau)$ is re-planned. It can be confirmed that, from $t=1$ to $t=2.5$, collisions with the obstacle within the time interval $t_w=0.2$ are predicted and $(z, \tau)$ is re-planned multiple times.

\section{Conclusion and Discussion}
We have proposed a novel model of movement primitives based on the motion manifold hypothesis, named Motion Manifold Primitives++ (MMP++). This model synergizes the strengths of the existing MMP with those of conventional parametric curve representation-based primitive models, achieving enhanced motion generation accuracy and adaptability. Moreover, to overcome the geometric distortion issue in MMP++, we have introduced a CurveGeom Riemannian metric for the parametric curve space and presented Isometric Motion Manifold Primitives++ (IMMP++). This approach ensures that the latent space preserves the geometric structure of the motion manifold, which, in some instances, leads to significantly improved density fitting results within the latent space.

We highlight that the low dimensionality of the latent parametrization of the motion manifold has led to multiple advantages. This allows us to exploit non-parametric density fitting techniques, such as kernel density estimation, in the latent space, which are typically challenging to apply in high-dimensional spaces, for example, directly in the curve parameter space ${\cal W}$. Additionally, re-planning has been easily accomplished by finding a new latent value $z$, formulated as an optimization problem that is efficiently solvable thanks to the low dimensionality of the latent space.

We have extended MMP++ to accommodate matrix Lie group data, specifically SO(3), by designing a parametric curve model that utilizes the exponential map, logarithm map, and group action. Given that the water-pouring SE(3) trajectory dataset forms a connected manifold without a clustering structure, MMP++ -- even without isometric regularization -- has shown satisfying performance. However, in instances where we encounter a dataset with clustering structures, where a distorted latent space could result in poor density fitting, as verified in the other 2-DoF and 7-DoF robot experiments, employing isometric regularization alongside an appropriate Riemannian metric can offer a solution. Although developing isometric regularization for matrix Lie group data falls outside the scope of our paper, it represents an interesting direction for future research.

Although our focus has been on via-point models, we can adopt other parametric curve models with stronger inductive biases. Even more sophisticated movement primitives, such as those based on dynamical systems, can be used, for example, by making the parameters of the dynamical systems the outputs of the neural network, as done in~\cite{bahl2020neural}. These can potentially be combined with the autoencoder-based manifold learning technique. These combinations will yield a variety of manifold-based motion primitives capable of generating diverse motions, demonstrating high adaptability in previously unseen dynamic environments.

Lastly, we note that our current motion manifold primitives are not conditioned on vision inputs such as images, point clouds, or sequences of images, so the learned manifold of trajectories is applicable only to the trained environment, such as the specific shelf in Fig.~\ref{fig:manifol_results}. If the MMP model could generalize to environments it has not encountered during training, it would be extremely powerful. For example, if arbitrary shelf geometries were given as vision inputs, and the model could generate a manifold of diverse trajectories conditioned on those inputs, it would significantly enhance its applicability. Adopting techniques (e.g., network architectures) from recent LfD methods for visuomotor policy learning~\cite{chi2023diffusion, lee2024behavior} is a feasible direction to extend our framework to generate a manifold of trajectories conditioned on vision inputs. For example, the decoder can be modified to take a vision feature \(y\) along with the latent variable \(z\), so that the decoder maps \((z, y)\) to the curve parameter \(w\). Although this would require a relatively large amount of data paired with visual observations, our core idea remains applicable.

\section*{Acknowledgments}
Yonghyeon Lee was the beneficiary of an individual grant from CAINS supported by a KIAS Individual Grant (AP092701) via the Center for AI and Natural Sciences at Korea Institute for Advanced Study.




 
%

\bibliographystyle{IEEEtran}
\bibliography{ref}

\section*{Biography}
 

\vspace{-33pt}
\begin{IEEEbiography}[{\includegraphics[width=1in,height=1.25in,clip,keepaspectratio]{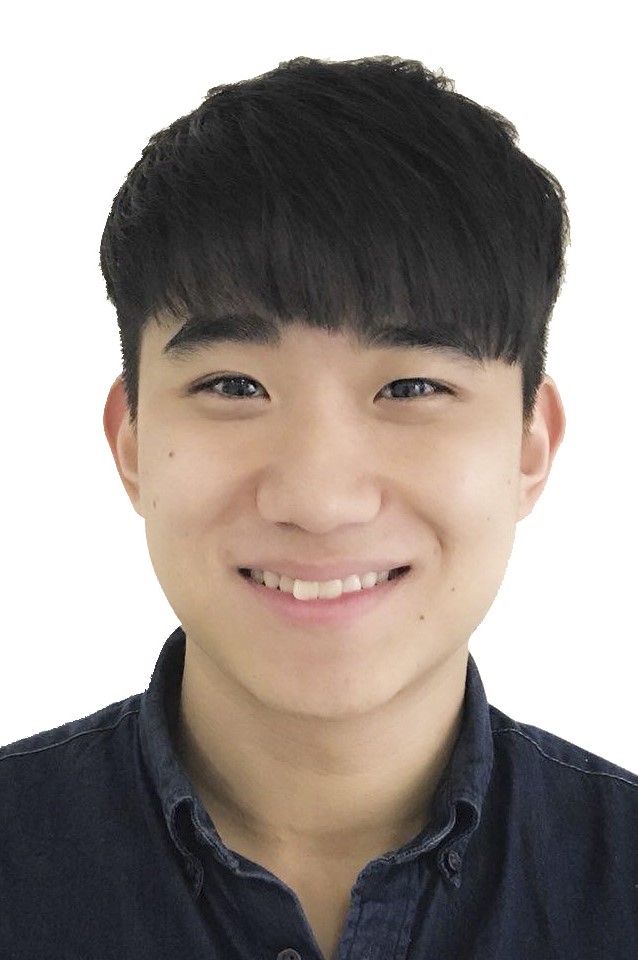}}]{Yonghyeon Lee} (Member, IEEE) received his B.S. degree in Mechanical Engineering in 2018 and his Ph.D. degree in Mechanical Engineering in 2023 from Seoul National University. He has been an AI Research Fellow at the Center for Artificial Intelligence at the Korea Institute for Advanced Study since 2023. His research interests include representation learning, robot learning, deep learning, dynamics, planning and control, vision and image processing, and geometric machine learning.

\end{IEEEbiography}




\end{document}